\documentclass[a4paper,fleqn]{cas-dc}

\usepackage[authoryear,longnamesfirst]{natbib}

\def\tsc#1{\csdef{#1}{\textsc{\lowercase{#1}}\xspace}}
\tsc{QILP-0}

\usepackage{amsmath,amssymb,amsthm}
\usepackage{algorithm}
\usepackage{booktabs}

\usepackage{graphicx}
\usepackage{booktabs}
\usepackage{url}  
\usepackage{hyperref}

\usepackage[section]{placeins}
\usepackage{float}

\usepackage{amsmath}
\usepackage{amssymb}
\usepackage{mathtools}
\usepackage{stmaryrd}

\usepackage{relsize}
\usepackage{graphicx}
\usepackage{subcaption}
\usepackage{amsmath}
\usepackage{mathtools}
\usepackage{algorithm}
\usepackage[noend]{algpseudocode}
\usepackage{xcolor,colortbl}
\usepackage{amssymb}
\usepackage{moresize}
\usepackage{multirow}
\usepackage{mdframed}
\usepackage{stmaryrd}
\usepackage{setspace}
\usepackage{enumitem}
\usepackage{comment}

\newcounter{example}[section]
\newenvironment{example}[1][]{\refstepcounter{example}\par\medskip
	\noindent \textbf{Example~\theexample} \rmfamily}{\medskip}

\newcounter{definition}[section]
\newenvironment{definition}[1][]{\refstepcounter{definition}\par\medskip
	\noindent \textbf{Definition~\thedefinition~(#1)} \rmfamily}{\medskip}

\newcounter{proposition}[section]
\newenvironment{proposition}[1][]{\refstepcounter{proposition}\par\medskip
	\noindent \textbf{Proposition~\theproposition} \rmfamily}{\medskip}

\newcounter{theorem}[section]
\newenvironment{theorem}[1][]{\refstepcounter{theorem}\par\medskip
	\noindent \textbf{Theorem~\thetheorem} \rmfamily}{\medskip}
	
\theoremstyle{remark}
\newtheorem{remark}[theorem]{Remark}
	
\newtheorem{corollary}[theorem]{Corollary}

\usepackage{prettyref}
\newrefformat{def}{Definition~\ref{#1}}
\newrefformat{prop}{Proposition~\ref{#1}}
\newrefformat{th}{Theorem~\ref{#1}}
\newrefformat{algo}{Algorithm~\ref{#1}}
\newrefformat{fig}{Figure~\ref{#1}}
\newrefformat{tab}{Table~\ref{#1}}
\newrefformat{ex}{Example~\ref{#1}}
\newcommand{\pref}{\prettyref}

\definecolor{gray50}{gray}{0.45}

\newcommand{\ie}{i.e., }
\newcommand{\resp}{resp.\ }

\newcommand{\N}{\ensuremath{\mathbb{N}}}
\newcommand{\segm}[2]{\ensuremath{\llbracket #1 ; #2 \rrbracket}}
\newcommand{\card}[1]{\ensuremath{|#1|}}
\newcommand{\powerset}{\ensuremath{\wp}}

\newcommand{\mvl}{\ensuremath{\mathcal{M}\mathrm{V}\mathrm{L}}}
\newcommand{\mvlp}{\ensuremath{\mathcal{M}\mathrm{V}\mathrm{L}\mathrm{P}}}
\newcommand{\V}{\ensuremath{\mathcal{V}}}
\newcommand{\var}{\ensuremath{\mathrm{v}}}

\newcommand{\val}{\ensuremath{{val}}}
\newcommand{\vvi}[1]{\ensuremath{\var_{#1}^{\val_{#1}}}}
\newcommand{\vv}{\ensuremath{\var^{\val}}}

\newcommand{\hvar}[1]{\ensuremath{\mathrm{var}({h({#1})})}}
\newcommand{\bvar}[1]{\ensuremath{\mathrm{var}({b({#1})})}}
\newcommand{\setvar}[1]{\ensuremath{\mathrm{var}({#1})}}
\newcommand{\dom}{\ensuremath{\mathsf{dom}}}
\newcommand{\A}{\ensuremath{\mathcal{A}}}
\newcommand{\Sall}{\ensuremath{\mathcal{S}}}

\newcommand{\dmvlp}{\ensuremath{\mathcal{D}\mathcal{M}\mathrm{V}\mathrm{L}\mathrm{P}}}
\newcommand{\F}{\ensuremath{\mathcal{F}}}
\newcommand{\T}{\ensuremath{\mathcal{T}}}
\newcommand{\Val}{\ensuremath{\mathcal{V}al}}
\newcommand{\SallF}{\ensuremath{\Sall^{\mathcal{F}}}}
\newcommand{\SallT}{\ensuremath{\Sall^{\mathcal{T}}}}

\begin{document}
\let\WriteBookmarks\relax
\def\floatpagepagefraction{1}
\def\textpagefraction{.001}

\shorttitle{QILP-0: observational declarative twins for quantum circuits}
\shortauthors{Author et al.}

\title[mode=title]{QILP-0: Constructing Observational Declarative Twins of Quantum Circuits}

\author[1]{Marina de la Cruz Echeandía}[
orcid=0000-0002-3038-5541
]
\ead{marina.delacruz@unir.net}
\credit{Software, Validation, Writing}

\affiliation[1]{
	organization={Universidad Internacional de la Rioja UNIR},
	addressline={Escuela Superior de Ingeniería y Tecnología},
	city={Logroño},
	country={Spain}
}

\author[3]{César Luis Alonso}
\ead{calonso@uniovi.es}
\credit{Conceptualization, Methodology, Validation, Writing}

\author[2]{Tony Ribeiro}[
orcid=0000-0002-1793-2854
]
\ead{tony.ribeiro@ls2n.fr}
\credit{Software, Validation, Writing}

\affiliation[2]{
	organization={Nantes Université, École Centrale Nantes, CNRS, LS2N, UMR 6004, 44000, Nantes, France;
		National Institute of Informatics, 2-1-2 Hitotsubashi, Chiyoda-ku, Tokyo, 101-8430, Japan}
}

\author[3]{Alfonso Ortega de la Puente}
\cormark[1]
\ead{ortegaalfonso@uniovi.es}
\credit{Conceptualization, Methodology, Software, Validation, Writing}

\affiliation[3]{
	organization={Universidad de Oviedo},
	addressline={Departamento de Informática},
	city={Gijón},
	postcode={33204},
	country={Spain}
}

\cortext[cor1]{Corresponding author}

\begin{abstract}
	This paper introduces QXymb, a general framework for constructing
	observational declarative twins of quantum circuits, and develops QILP-0,
	its first complete order-0 specialization. QILP-0 constructs a finite
	multi-valued propositional logic program from observed circuit behaviour
	within a declared observational scope.
	
	The pipeline traverses a declared family of quantum observables
	incrementally according to a reproducible structural grading and a declared
	observational reference horizon. Progress is quantified through
	reference-relative coverage against a fixed target-independent reference.
	Observable responses are organized through target-independent geometry,
	while retained latent structure is mapped deterministically back to original
	observable columns before symbolic processing, preserving observational
	semantics and provenance.
	
	Selected observable profiles are converted into a finite relation through
	admissible target-independent discretization. The target is used only
	afterwards to audit twin-admissibility and induce the declarative theory.
	A theory is certified as an exact observational declarative twin when it
	completely and correctly reconstructs the resulting finite task-conditioned
	discrete relation. Logical exactness is therefore separated from numerical,
	backend, provider, and discretization uncertainty, which is retained as
	audit metadata.
	
	Validation uses two complementary QML settings. Exhaustive Bars \& Stripes
	experiments compare product and grid-CZ embeddings from 16 to 100 qubits and
	exercise the native-discrete branch. Low-Depth MNIST analyses all 14,708
	digit-0/1 instances before and after a trained variational quantum
	transformation and exercises continuous discretization. In every reported
	relation, the induced QILP-0 theory achieves complete, conflict-free
	reconstruction with strict accuracy equal to one.
\end{abstract}

\begin{highlights}
	\item A target-independent pipeline links quantum observables and logical induction.
	\item Reference-relative coverage tracks complete exact-support layers.
	\item Latent geometry maps reproducibly back to original quantum observables.
	\item A declarative explainability approach for observed quantum-circuit behaviour.
\end{highlights}

\begin{keywords}
quantum machine learning \sep symbolic artificial intelligence \sep inductive logic programming \sep explainable machine learning \sep quantum observables
\end{keywords}

\maketitle

\section{Introduction}
\label{sec:introduction}

Quantum machine learning (QML) combines quantum information processing with
learning procedures in order to construct representations, kernels,
classifiers, or trainable quantum models. This combination introduces
explanatory questions that are not exhausted by the interpretability of the
classical learning component. A classical input may first be transformed by a
quantum embedding, subsequently processed by a parametrized quantum circuit,
and finally exposed only through a restricted set of measurements. Even when
the surrounding learning algorithm is familiar, the quantum representation
itself can therefore become an additional source of opacity
\cite{PiraFerrie2024,GilFusterEtAl2024}.

Two common QML settings illustrate this issue. In quantum-kernel and related
hybrid approaches, a classical datum is encoded into a quantum state before a
classical learning stage operates on similarities or measurements derived from
that representation. The embedding can substantially reorganize the geometry
seen by the downstream learner. In variational approaches, a trainable quantum
transformation further changes that representation according to a supervised
objective. In both cases, inspecting only the final prediction leaves open a
different question: what observable relation has the quantum part of the
pipeline actually produced over the analysed data?

Most explainability methods address related but different objects, such as
feature attribution, gate relevance, local surrogate behaviour, visual
inspection, or intrinsically interpretable architectures. Recent work has also
shown that quantum data can support the discovery of compact and physically
meaningful latent representations followed by symbolic descriptions
\cite{deSchoulepnikoffEtAl2026QDisc}. These approaches demonstrate the growing
importance of interpretable representations for quantum data. The present work
pursues a complementary objective: rather than replacing the observational
vocabulary by a compact learned representation before symbolic reasoning, we
ask whether an explicitly scoped part of the observable behaviour of a quantum
circuit can be reconstructed as a finite, traceable declarative theory.

This question motivates QXymb, a methodological bridge between quantum
observations and symbolic reasoning. QXymb is intended to accommodate different
observable providers, numerical backends, discretization policies and
declarative engines. The present paper develops QILP-0, its current
order-0 specialization, in which the symbolic layer is a finite multi-valued
propositional logic program induced through the LFIT/PRIDE family.

The central object introduced in this work is an \emph{observational
	declarative twin}. The adjective observational is essential. QILP-0 does not
claim to reconstruct the complete quantum behaviour of a circuit for every
possible input, state, measurement, or execution condition. Instead, the
researcher declares an observational scope comprising a dataset, an observable
family, a reproducible structural grading, and an observational reference
horizon. QILP-0 then constructs a finite declarative representation of the
relation observed within that scope and records the provenance required to
trace symbolic literals back to discrete states and original quantum
observables.

For the Pauli realization used in this work, the structural family can extend
up to full \(n\)-qubit support, but a particular execution may declare a
smaller observational reference horizon according to scientific scope,
available data, provider capabilities, or technological and computational
constraints. Coverage statements are explicitly relative to that declared
reference. They therefore do not estimate the fraction of unobserved
higher-support geometry lying outside it.

This scope qualification also determines how exactness is understood.
Logical exactness is assessed with respect to the finite discrete relation
actually produced by the pipeline. If that relation is target-consistent and
the induced program covers every analysed row without conflicting predictions,
the resulting theory is an exact observational declarative twin of that
reported relation. Any uncertainty introduced earlier by sampling, numerical
approximation, a backend, or a provider concerns the correspondence between the
reported observable responses and their ideal counterparts and must be
recorded separately. Thus, QILP-0 does not hide uncertainty inside the
symbolic claim.

A second design requirement is that supervised information should not shape
the observational vocabulary before declarative induction. Observable
generation, structural traversal, geometric analysis, coverage estimation,
selection of original observable columns and discretization are therefore
performed independently of the target. The target enters only when the finite
relation is audited for twin-admissibility and the declarative theory is
induced. This separation is intended to prevent the explanation vocabulary
from being engineered retrospectively around the classes that it is later
asked to describe.

The continuous observational layer is incremental. In the present
instantiation, Pauli observables are graded by exact support and processed in
complete blocks. Their response profiles are analysed through a
target-independent SVD-based geometry, while progress is measured by
reference-relative coverage against the fixed observational reference defined
for the declared horizon. The geometric stage may reveal latent directions,
but those directions do not replace the symbolic variables. QILP-0 evaluates
the association of the original observable columns with the retained new
subspace and deterministically preserves the score-ranked set required by the
declared column-association threshold. Discretization and declarative induction
therefore operate on original observational variables rather than on latent
singular coordinates.

The declarative layer then converts the selected observable profiles into a
finite symbolic relation through an admissible target-independent
discretization. Native discrete states are preserved; genuinely continuous
columns are discretized at a resolution constrained by the available
statistical, numerical and, when present, backend information. Once
twin-admissibility has been established, PRIDE induces the propositional theory
used by QILP-0 and the pipeline verifies its row-level reconstruction. The
result is accompanied by a certificate recording the observational scope,
discretization audit, symbolic consistency, reconstruction metrics and
provenance.

The contributions of this paper are therefore the following:

\begin{enumerate}
	\item We define the \emph{observational declarative twin} as a
	scope-qualified finite logical reconstruction of an observed quantum
	relation, distinguishing logical exactness from the uncertainty of the
	underlying observable evaluation.

	\item We formulate the methodological conditions required to construct such
	a twin, including semantically traceable observable families, reproducible
	structural grading, a declared observational reference horizon,
	target-independent geometry, reference-relative incremental coverage,
	deterministic mapping from retained latent geometry back to original
	observable columns, admissible target-independent discretization,
	twin-admissibility, and a declarative engine satisfying the required
	reconstruction contract.

	\item We provide a constructive QILP-0 realization of these conditions.
	The current implementation uses Pauli observables, exact-support traversal,
	a fixed observational reference, an SVD-based continuous geometry,
	deterministic original-observable selection, target-independent
	observable-wise discretization, and LFIT/ PRIDE induction, while keeping the
	corresponding components conceptually replaceable inside QXymb.

	\item We make the construction auditable through explicit provenance and
	certification. Symbolic literals remain traceable to discrete states,
	numerical intervals or native values, original observables, supports and
	qubits, and exact-twin claims are issued only after consistency and
	row-level reconstruction checks.
	
	\item We validate the construction in two complementary QML settings.
	An exhaustive Bars \& Stripes study exercises the native-discrete branch
	on fixed product and grid-CZ embeddings from 16 to 100 qubits. A
	Low-Depth MNIST study exercises the continuous branch on all 14,708
	digit-0/1 instances and compares the same representation before and after a
	trained variational transformation. In every reported relation, the resulting
	QILP-0 theory provides complete, conflict-free reconstruction with strict
	reconstruction accuracy equal to one.

\end{enumerate}

The experiments are intended to validate the methodological construction
rather than to benchmark predictive quantum advantage. Bars \& Stripes isolates
the embedding stage as a controlled source of representational change, whereas
Low-Depth MNIST allows the observational geometry and declarative relation to
be compared before and after supervised variational processing. Together they
exercise the two branches of the declarative interface---native discrete and
continuous---and show how the same certificate semantics applies to both.

The remainder of the paper is organized as follows.
Section~\ref{sec:background} reviews explainability in QML, logical
formalizations of quantum computation, interpretable representation learning
from quantum data, and the LFIT family underlying the current declarative
engine. The following methodological sections introduce QXymb/QILP-0, formalize
the continuous and declarative conditions required by the construction, define
the observational declarative twin, and give its constructive realization and
certificate. Section~\ref{sec:results} reports the Bars \& Stripes and
Low-Depth MNIST experiments together with their cross-experiment
certification. The final section summarizes the conclusions, limitations and
main directions for future work.

\section{Background and Related Work}
\label{sec:background}

The context of this contribution includes current approaches to explainability
in QML, logical and formal descriptions of quantum computation, and methods for
inducing equivalent logical theories from processes represented as datasets.

\subsection{Explainability in QML}
\label{sec:related-xqml}

QML models that wrap classical ML engines with a quantum embedding level can
benefit from the same explainability tools used for their classical ML
component. These approaches are outside the scope of the present contribution,
which is centred on specifically quantum explainability and interpretability
approaches.

Current QML models for classical data, such as variational quantum circuits
(VQCs), commonly combine a data-encoding circuit, a trainable parametrized
circuit and one or more measurements. Their hybrid structure makes them
amenable to some classical explainability tools, but it also creates
specifically quantum difficulties. Intermediate quantum states are not
generally available as reusable layer activations; exact state descriptions
scale exponentially; measurements are probabilistic; and finite-shot noise can
affect both predictions and their explanations. Consequently, explainability
techniques developed for classical neural networks cannot always be transferred
without modification \cite{PiraFerrie2024,GilFusterEtAl2024}.

The emerging literature in this field can be organized into four complementary
families.

The first adapts \emph{post-hoc feature attribution and local surrogate
	methods} to quantum classifiers. Q-LIME, for example, extends local
model-agnostic explanation to quantum neural networks and explicitly considers
the randomness introduced by quantum measurements \cite{PiraFerrie2024}. Other
studies combine occlusion, gradient-based attribution and example influence to
explain QNN predictions \cite{TianYang2024}, or wrap quantum classifiers with
established LIME and SHAP procedures to obtain local and global feature-level
accounts \cite{KadianEtAl2025,DonKhalil2025}. Recent work has also begun to
examine the stability of local explanations for quantum classifiers
\cite{AcamporaVitiello2026}.

A second family attributes relevance to \emph{internal circuit components}
rather than to input features. Heese et al.\ adapt Shapley values to quantify
the contribution of gates or groups of gates to a task-dependent objective
\cite{HeeseEtAl2025}. These explanations are useful for circuit diagnosis and
design, but they remain attribution scores rather than declarative
descriptions of the observed input--output relation.

A third line develops \emph{quantum-aware analytical or intrinsically
	interpretable models}. Gil-Fuster et al.\ formulate a broader framework for
explainable QML and propose techniques tailored to parametrized quantum
circuits \cite{GilFusterEtAl2024}. Complementarily, concept-driven QNNs
introduce a human-interpretable concept layer into the model itself
\cite{TianYangConcept2024}. Such approaches are promising, although intrinsic
interpretability usually requires architectural commitments that are not
available when an already trained circuit must be analysed.

A fourth family treats explainability as a \emph{visual-analytics problem}.
VIOLET separates the encoding, ansatz and learned-feature views of a QNN and
combines circuit structure, parameter evolution and measurement information in
an interactive environment \cite{RuanEtAl2023}. This illustrates that QML
explanation may require coordinated views of several stages of the
quantum--classical pipeline rather than a single attribution vector.

Several adjacent research lines also deserve consideration when assessing the
scope of a symbolic approach. Classical surrogates can reproduce the
input--output behaviour of a quantum learning model, and shadow models can
transfer information obtained from quantum experiments to an efficiently
deployable classical predictor
\cite{SchreiberEisertMeyer2023,JerbiEtAl2024}. Their primary output is,
however, another predictive model rather than a logical theory whose literals
are directly traceable to selected observables. Other work infers compact
dynamical generators from trajectories of local quantum observables
\cite{CeminEtAl2024}; this produces interpretable equations for dynamics, but
does not induce a class-target rule theory from a dataset of circuit
executions. Finally, formal quantum logics, symbolic verification and quantum
logic programming use logical languages to specify, verify or execute quantum
processes \cite{BauerMarquartEtAl2023,Brunet2016,Balu2015}. Their direction is
therefore different from learning a logical description from the observed
behaviour of an already defined circuit.

A particularly close and rapidly developing line of work uses representation
learning to extract physically meaningful low-dimensional descriptions from
quantum data. Earlier work on operationally meaningful representations showed
that neural representations can be constrained so that their latent factors
retain an explicit physical interpretation, including compact
representations of two-qubit states that separate local information from
quantum correlations \cite{NautrupEtAl2022Operational}. More recently,
probabilistic variational autoencoders have been adapted to the intrinsic
randomness and correlations of quantum measurement data, with the explicit
goal of learning compact and physically interpretable latent representations
without prior labels or known order parameters
\cite{deSchoulepnikoffEtAl2025ProbVAE}. Related work on quantum-simulator
snapshots likewise uses unsupervised variational autoencoders to discover
minimal latent representations correlated with physically relevant variables
\cite{MollerEtAl2026MinimalManyBody}, while action-induced representations
provide a complementary route in which latent components are associated with
physical degrees of freedom through the experimental actions that affect them
\cite{MunozGilEtAl2026AIR}.

The recent QDisc framework brings this programme especially close to the
objectives considered here
\cite{deSchoulepnikoffEtAl2026QDisc}. QDisc processes quantum measurement data
with a probabilistic variational autoencoder, identifies structure in the
resulting compact latent representation, and subsequently applies symbolic
regression to obtain compact analytical descriptors that can act as order
parameters for the discovered regimes. It has been demonstrated on
experimental Rydberg-atom data, classical-shadow data and fermionic datasets.
The combination of interpretable representation learning and symbolic
regression makes QDisc an important neighbouring approach for
data-driven symbolic discovery in quantum systems.

The methodological objective of QILP-0 is nevertheless different. The
representation-learning line above deliberately searches for a compact latent
description containing the physically relevant factors needed to characterize
or reconstruct the observed phenomena. QILP-0, by contrast, does not replace
the declared observational vocabulary by learned latent coordinates before
symbolic induction. Its geometric stage organizes the observable-response
structure in a target-independent manner while preserving the association with
the original observables; the selected original observable profiles are then
discretized under an explicit observational-resolution contract, and logical
induction determines which conditions over those observable states are
sufficient to reconstruct the analysed relation. In this sense, compact
representation learning and QILP-0 place semantic simplification at different
points of the pipeline: the former deliberately compresses the representation
before symbolic description, whereas QILP-0 preserves the original
observational semantics through the representation-building stages and allows
the declarative theory to express the subsequent simplification.

The two approaches therefore address complementary scientific questions.
QDisc asks which compact, physically meaningful latent factors and analytical
expressions characterize structure discovered in quantum data. QILP-0 asks
which finite declarative theory reconstructs an explicitly scoped observed
relation while keeping every symbolic literal traceable to the observable
vocabulary from which that relation was constructed. This distinction concerns
the methodological object being sought rather than a preference for one
learning algorithm over another.

Within this landscape, QILP-0 targets a complementary level of explanation. It
is a post-hoc, dataset-level and representation-oriented pipeline. Instead of
assigning a relevance score only to original input features or individual
gates, it evaluates a declared family of quantum observables, organizes them
incrementally by structural support, and analyses their response geometry in a
target-independent manner. Progress is measured relative to a fixed
observational reference, and the retained latent geometry is mapped
deterministically back to original observable columns before target-independent
discretization and declarative induction.

The researcher controls the observational scope through a declared reference
horizon and an explicit stopping policy. For the Pauli realization considered
in this work, the structural support range extends up to the number of qubits,
but a particular execution may declare a smaller reference horizon according
to scientific scope, available observational data, provider capabilities, or
computational resources. QILP-0 reports coverage relative to that declared
reference and separately records any downstream endpoint reached before the
reference horizon is fully propagated. It therefore does not extrapolate its
coverage claim to unobserved higher-support structure.

The resulting records connect the observational geometry of the quantum
representation, the original observables and qubit supports involved, and a
symbolic account of the observed relation to the target. We call the resulting
object, formally defined in the methodological sections, an
\emph{observational declarative twin}. Every rule literal remains traceable to
an original observable, its qubit support, and a discrete state.

The distinctive contribution of QILP-0 lies in combining support-wise
target-independent observational traversal, reference-relative coverage
against a fixed declared reference, deterministic return from latent geometry
to original observable variables, admissible target-independent
discretization, logical rule induction, and an explicit certificate of
equivalence with the resulting finite observed relation. To the best of our
knowledge, we have not identified a previous framework that combines these
elements into a certified observational declarative twin with this scope and
provenance contract. Closely related approaches address feature or gate
attribution, predictive classical surrogates, compact interpretable latent
representations and symbolic physical descriptors, interpretable observable
dynamics, or logic-based specification and verification. These directions are
complementary to the finite declarative reconstruction pursued here.

\subsection{Formalization of quantum computations}
\label{sec:logic_quantum_circuits}

A complementary line of work, distinct from explainable QML, has investigated
the logical formalization and verification of quantum computation itself.
Quantum dynamic logics provide an important precedent in this direction.
Baltag and Smets' Logic of Quantum Programs (LQP) represents quantum
measurements, unitary transformations, locality and entanglement within a
dynamic-logical semantics, enabling properties of quantum programs and
protocols such as teleportation and quantum secret sharing to be specified and
proved formally \cite{BaltagSmets2006LQP}. Subsequent work extended this line
towards probabilistic reasoning and decidable logics for quantum algorithms,
including formal encodings of quantum search and distributed protocols
\cite{BaltagEtAl2014PLQP}, while a later overview consolidated quantum dynamic
logic as a framework for reasoning about the flow of quantum information and
for verifying quantum protocols \cite{BaltagSmets2022QDLOverview}.

This tradition is conceptually relevant to QXymb because it demonstrates that
quantum programs can be connected to explicit logical objects while
preserving information about their quantum structure. The direction pursued in
our present contribution is, however, essentially inductive rather than
deductive: instead of starting from a logical specification and proving that a
quantum program satisfies it, QILP-0 starts from the observable behaviour of an
existing circuit and induces a declarative theory that reconstructs that
behaviour within a declared observational scope. We therefore regard quantum
program logics as an important conceptual foundation and a complementary
formal layer. In future developments, this connection may also be exploited at
the QXymb level, for example by associating structured circuit
representations with formal specifications or verification interfaces, while
QILP-0 remains responsible for learning declarative descriptions from observed
circuit behaviour.

\subsection{Induction of propositional theories equivalent to datasets}
\label{sec:soa_ILP_propositional}

Although different inductive engines can synthesize declarative models from
examples and counterexamples represented in datasets, for reasons of space we
focus here on the family actually used by QILP-0: Learning From Interpretation
Transition (LFIT).

LFIT \cite{InoueRibeiroSakama2014} was proposed to automatically construct a model of the
dynamics of a system from observations of its state transitions. Given raw
data, such as time-series gene-expression data, a discretization of those data
in the form of state transitions is assumed. From those state transitions,
according to the semantics of the system dynamics, several inference
algorithms modelling the system as a logic program have been proposed. The
semantics of a system's dynamics can differ with regard to the synchronism of
its variables, the determinism of its evolution and the influence of its
history.

The LFIT framework proposes several modelling and learning algorithms to
tackle those different semantics. To date, the following systems have been
addressed: memoryless deterministic systems \cite{InoueRibeiroSakama2014}, systems with
memory \cite{TRFrontier15}, probabilistic systems \cite{DMTRICLP15} and their
multi-valued extensions \cite{TRICMLA15,DMTRICAPS16}. The work
\cite{TRILP2017} proposes a method that allows continuous time-series data to
be handled, with the abstraction itself learned by the algorithm.

In \cite{TRILP18,TRMLJ2020}, LFIT was extended to learn system dynamics
independently of its update semantics. That extension relies on a modelling of
discrete memoryless multi-valued systems as logic programs in which each rule
represents that a variable can take some value at the next state, extending
the formalism introduced in \cite{InoueRibeiroSakama2014,TRILP14}. The representation in
\cite{TRILP18,TRMLJ2020} is based on annotated logics
\cite{Blair1989135,blair1988paraconsistent}. Here, each variable corresponds
to a domain of discrete values. In a rule, a literal is an atom annotated with
one of these values. This allows annotated atoms to be represented as
classical atoms and hence preserves a propositional representation.

This modelling makes it possible to characterize optimal programs
independently of the update semantics and to represent the dynamics of a wide
range of discrete systems, including the finite multi-valued relation required
by QILP-0. LFIT can therefore be used to induce an equivalent propositional
logic program that provides a declarative explanation for each supplied
observation. The specific reconstruction properties required by QILP-0 and
their instantiation through PRIDE are formalized later as part of the
observational-twin contract rather than repeated here.

\section{Continuous observational layer: methodological conditions for target-independent construction}
\label{sec:observational-framework}

The constructive result developed later in this paper relies on a sequence of
methodological conditions that must be made explicit before the declarative
stage is defined. QILP-0 must first be able to represent the observed
behaviour of a quantum circuit numerically, interrogate that behaviour through
a semantically meaningful observable family, traverse the family according to
a reproducible structural grading, quantify the observational structure
revealed during that traversal, and stop under explicit conditions without
using the supervised target to guide the construction. Finally, the numerical
geometry used to organize the exploration must be mapped back to the original
observable vocabulary before symbolic induction.

The following subsections state these conditions and provide the mathematical
and bibliographic support used by the current QILP-0 construction. 
All guarantees remain relative to the declared observational
scope: the analysed inputs, observable provider, structural horizon,
backend or estimator, numerical tolerances, and subsequent discretization
policy.

\subsection{Condition 1: an observational behaviour matrix can be constructed}
\label{subsec:condition-observational-matrix}

The starting condition of QILP-0 is that the behaviour of the quantum system
under a finite collection of semantically identifiable input situations can
be represented through a real-valued observational dataset. This numerical
representation constitutes the continuous layer of the current methodology:
individual observable profiles may in practice take either continuously
varying values or a finite set of native values, but they are represented
before symbolic discretization as numerical columns in a real matrix.

Let \(m\) be the number of analysed input situations and let
\(\{x_i\}_{i=1}^{m}\) denote their classical descriptions. Classical inputs
can be embedded into the Hilbert space of a quantum system through
data-dependent quantum feature maps
\cite{schuldkilloran2019feature}. After the data-dependent preparation and any
subsequent quantum processing, let

\begin{equation}
	\rho(x_i)\in\mathbb C^{2^n\times 2^n}
\end{equation}

denote the resulting density operator for input \(x_i\). Thus,
\(\rho(x_i)\succeq0\) and
\(\operatorname{Tr}(\rho(x_i))=1\).

Let \(\{O_j\}_{j=1}^{p}\) be a declared family of \(p\) Hermitian observables
acting on the same Hilbert space. Quantum learning models commonly expose
properties of their processed states through expectation values of chosen
observables \cite{mitarai2018quantum}. QILP-0 therefore defines

\begin{equation}
	\begin{split}
	(X_{\mathrm{obs}})_{ij}
	=
	X_{ij}
	=
	\operatorname{Tr}\!\left[\rho(x_i)O_j\right],
	\qquad \\
	i=1,\ldots,m,\quad
	j=1,\ldots,p,
	\label{eq:qilp-observational-entry}
	\end{split}
\end{equation}

with
\(X_{\mathrm{obs}}\in\mathbb R^{m\times p}\).

Because \(O_j\) is Hermitian, every \(X_{ij}\) is real. Row \(i\) of
\(X_{\mathrm{obs}}\) is the observable profile produced by the quantum system
for input \(x_i\), whereas column \(j\) records the response of observable
\(O_j\) across the complete analysed dataset.

\paragraph{Intuition.}
The observational matrix can be viewed as a table of questions and answers.
Each observable asks one fixed question about the processed quantum state.
Each row corresponds to one input situation and records the answers obtained
for that situation. QILP-0 does not immediately collapse this table into a
prediction or into latent coordinates: the observable responses remain
individually identifiable because the later declarative description must be
able to refer back to what was actually observed.

Unlike approaches that immediately aggregate expectation values into a single
prediction, QILP-0 preserves the original observable profiles and their
identifiers. A separate target, when the analysed task is supervised, is not
used to construct \(X_{\mathrm{obs}}\). Once the analysed quantum producer is
fixed, QILP-0 does not access that target during Conditions~1--8; its first
access occurs in the twin-admissibility audit of Condition~9.

This first condition therefore does not claim that \(X_{\mathrm{obs}}\)
contains every physically available property of the circuit. It establishes
a finite observational representation of the circuit behaviour for the
declared inputs and observable family. The expressive adequacy of that
observable family is the second methodological condition.

\subsection{Condition 2: a semantically adequate observable family can be declared}
\label{subsec:condition-observable-family}

The first condition is deliberately agnostic about the internal structure of
the quantum behaviour being represented. It only establishes that, for a
declared collection of input situations, the circuit can be associated with a
real-valued observational dataset. The second condition asks a stronger
question: whether there exists a family of observables rich enough to provide
a complete description of the quantum state underlying those observations.

For an \(n\)-qubit system, such a possible family could be provided by the Pauli operators.
Their expectation values can be regarded as coordinates of the density
operator in a complete orthogonal operator basis. This makes the Pauli family
a particularly well-founded observational language for QILP-0: rather than
assuming a particular internal organization of the circuit behaviour, the
methodology starts from a basis capable, in the non-truncated case, of
expressing the complete state from which that behaviour is observed.

The QILP-0 observable provider used in the present work is the Pauli family.

For an \(n\)-qubit system, define
\(\sigma_0=I\), \(\sigma_1=X\), \(\sigma_2=Y\), and
\(\sigma_3=Z\). For every multi-index

\begin{equation}
	\boldsymbol{\alpha}
	=
	(\alpha_1,\ldots,\alpha_n)
	\in
	\{0,1,2,3\}^{n},
\end{equation}

let

\begin{equation}
	P_{\boldsymbol{\alpha}}
	=
	\sigma_{\alpha_1}\otimes\cdots\otimes\sigma_{\alpha_n},
	\qquad
	\mathcal P_n
	=
	\left\{
	P_{\boldsymbol{\alpha}}:
	\boldsymbol{\alpha}\in\{0,1,2,3\}^{n}
	\right\}.
	\label{eq:qilp-pauli-family}
\end{equation}

Tensor products of single-qubit Pauli matrices form a complete operator basis
on the \(n\)-qubit Hilbert space
\cite{siewert2022orthogonal,lawrence2002mutually}. Under the
Hilbert--Schmidt inner product,

\begin{equation}
	\operatorname{Tr}
	\left(
	P_{\boldsymbol{\alpha}}^{\dagger}
	P_{\boldsymbol{\beta}}
	\right)
	=
	2^n
	\delta_{\boldsymbol{\alpha}\boldsymbol{\beta}},
	\label{eq:qilp-pauli-hs-orthogonality}
\end{equation}

where
\(\delta_{\boldsymbol{\alpha}\boldsymbol{\beta}}\)
is the Kronecker delta. Since
\(\lvert\mathcal P_n\rvert=4^n\), equal to the dimension
\((2^n)^2\) of the linear operator space, the family is complete.
Consequently, every \(n\)-qubit density operator admits the expansion

\begin{equation}
	\rho
	=
	\frac{1}{2^n}
	\sum_{\boldsymbol{\alpha}\in\{0,1,2,3\}^{n}}
	\operatorname{Tr}
	\left(
	\rho P_{\boldsymbol{\alpha}}
	\right)
	P_{\boldsymbol{\alpha}}.
	\label{eq:qilp-pauli-density-expansion}
\end{equation}

Thus, the complete collection of non-identity Pauli expectation values,
together with
\(\operatorname{Tr}(\rho)=1\),
determines the state.

\paragraph{Intuition.}
Condition 1 states that QILP-0 can describe an execution by recording answers
to observable questions. Condition 2 asks whether the available vocabulary of
questions is rich enough to support the intended description. The full Pauli
family gives a particularly strong answer: in principle, its expectation
values form a complete coordinate system for the density operator. QILP-0
does not need to claim that every execution will ask every one of those
questions; instead, the complete family provides a well-founded observational
language from which a finite observational reference horizon can be declared.

QILP-0 therefore uses Pauli observables as the provider in the experiments
reported here, but the methodology is not restricted to this choice.
An alternative observable provider is admissible when it exposes a
declared and semantically traceable family, stable identifiers, a reproducible
organization of that family, and the provenance required to interpret the
resulting values. When incremental traversal is intended, the provider must
also define the meaning of its incremental blocks and the observational
reference horizon against which later coverage will be interpreted.

The provider metadata must document its semantics, intended observational
scope, observable identifiers, and any approximation, truncation, sampling,
numerical error, or backend-induced uncertainty affecting the expectation
values. These metadata remain available to the later geometric,
discretization, and symbolic stages.

The completeness of the full Pauli basis therefore supports this provider
choice, whereas any actual QILP-0 claim remains relative to the observable
family and observational reference horizon that were actually declared and
evaluated.

\subsection{Condition 3: the observable family admits a reproducible structural grading}
\label{subsec:condition-structural-grading}

QILP-0 does not know in advance which observables will provide the most useful
description of the analysed behaviour. Its incremental exploration therefore
requires a structural grading that can be defined before any target-dependent
analysis is performed.

For the Pauli provider used here, QILP-0 uses Pauli support.
Given

\begin{equation}
	P_{\boldsymbol{\alpha}}
	=
	\sigma_{\alpha_1}\otimes\cdots\otimes\sigma_{\alpha_n},
\end{equation}

its support and Pauli weight are

\begin{equation}
	\begin{split}
		\operatorname{supp}
		\left(
		P_{\boldsymbol{\alpha}}
		\right)
		&=
		\left\{
		q\in\{1,\ldots,n\}:
		\alpha_q\neq0
		\right\},
		\\
		\operatorname{wt}
		\left(
		P_{\boldsymbol{\alpha}}
		\right)
		&=
		\left|
		\operatorname{supp}
		\left(
		P_{\boldsymbol{\alpha}}
		\right)
		\right|.
		\label{eq:qilp-pauli-support-weight}
	\end{split}
\end{equation}

This agrees with the standard interpretation of Pauli weight as the number of
qubits on which an operator acts non-trivially
\cite{ippoliti2024classicalshadows}. Locality is also operationally relevant:
under local Pauli measurements, estimation cost can depend strongly on the
active support
\cite{huang2020predicting,ippoliti2024classicalshadows}.

For each \(k\in\{1,\ldots,n\}\), define the exact-support block

\begin{equation}
	\mathcal O^{(k)}
	=
	\left\{
	P\in\mathcal P_n:
	\operatorname{wt}(P)=k
	\right\},
	\label{eq:qilp-exact-support-block}
\end{equation}

with cardinality

\begin{equation}
	\left|
	\mathcal O^{(k)}
	\right|
	=
	\binom{n}{k}3^k.
	\label{eq:qilp-exact-support-count}
\end{equation}

The factor
\(\binom{n}{k}\)
selects the active qubits and
\(3^k\)
selects \(X\), \(Y\), or \(Z\) on each active position.

\paragraph{Intuition.}
Pauli support provides an observational zoom. Support one asks questions that
act non-trivially on one qubit at a time; support two allows joint questions
over pairs; support three over triples; and so on. Increasing support therefore
opens the observation to progressively larger groups of qubits. This is the
sense in which the original manuscript informally associated the traversal
with increasing complexity: larger supports permit more articulated joint
observations. Support itself is not, however, a direct measure of physical
interaction, entanglement, causal influence, informativeness, or explanatory
importance.

QILP-0 evaluates each declared exact-support block once, in increasing order
of \(k\), without regenerating lower supports. A grade is counted as processed
only when its exact-support block has been completed. No target is used to
alter this order.

For the Pauli provider, the exact-support grading has a structural maximum
\begin{equation}
	K_{\mathrm{phys}}=n,
	\label{eq:qilp-physical-support-horizon}
\end{equation}
because no Pauli word can act non-trivially on more than the \(n\) qubits of
the circuit.

A QILP-0 execution additionally receives a declared observational reference
horizon
\begin{equation}
	1\leq K_{\mathrm{ref}}\leq K_{\mathrm{phys}}.
	\label{eq:qilp-reference-support-horizon}
\end{equation}
The corresponding observational family is
\begin{equation}
	\mathcal O_{\leq K_{\mathrm{ref}}}
	=
	\bigcup_{k=1}^{K_{\mathrm{ref}}}
	\mathcal O^{(k)}.
	\label{eq:qilp-observational-horizon}
\end{equation}

The value of \(K_{\mathrm{ref}}\) is part of the declared observational scope,
not a quantity inferred by the target-dependent stages of QILP-0. It may be
specified from domain requirements, determined by the observables available
from a dataset or provider, or selected through an independent technological
or resource policy. If the complete provider horizon is available and relevant,
one may choose \(K_{\mathrm{ref}}=K_{\mathrm{phys}}\).

Every later statement about geometric mass or coverage is relative to
\(\mathcal O_{\leq K_{\mathrm{ref}}}\). Therefore, when
\(K_{\mathrm{ref}}<K_{\mathrm{phys}}\), QILP-0 makes no quantitative claim
about the fraction of observable geometry lying beyond the declared reference
horizon and does not claim complete physical or informational
characterization of the unrestricted quantum state.

Pauli support is the grading used in the present experiments, not a
restriction of the framework.

Another provider may participate in the
incremental traversal when it supplies a semantically traceable observable
family, stable identifiers, a reproducible grading into finite exact-grade
blocks, and an explicit horizon. The meaning of that grading and any
provider-dependent uncertainty must be recorded as provenance.

Condition 3 therefore supplies a reproducible order for exploration. It does
not decide which blocks are empirically informative. That question is
delegated to the target-independent geometric analysis of the next condition.

\subsection{Condition 4: observable profiles admit a target-independent geometric representation}
\label{subsec:condition-geometric-representation}

The structural grading of Condition 3 provides a reproducible order in which
the observable family can be explored, but it does not quantify redundancy,
independent variation, or the amount of observational structure contributed
by the corresponding profiles. Carrying every observable directly to the
declarative stage would preserve the complete declared vocabulary, but would
provide no geometric account of these relationships and could unnecessarily
propagate redundant dimensions to subsequent stages.

The current QILP-0 construction therefore introduces a target-independent
geometric stage before declarative induction. In the instantiation studied in
this work, singular value decomposition (SVD) provides the concrete mechanism
used to organize redundancy, identify numerically independent directions, and
support the incremental coverage analysis developed below. The existence of a
declarative representation is not claimed to depend uniquely on this
particular decomposition; rather, SVD is the geometric realization adopted,
formalized, and experimentally validated in the present QILP-0 construction.

This geometric mechanism is not specific to the Pauli provider. Once an
observable provider yields a finite real-valued matrix whose columns remain
semantically identifiable, the same matrix-geometric analysis can in
principle be applied independently of the physical nature of the observables.
Provider-specific structure---such as Pauli support---determines how the
observables are generated and graded, whereas SVD operates on the resulting
observable-response profiles.

The original matrix
\(X_{\mathrm{obs}}\in\mathbb R^{m\times p}\)
is preserved unchanged for subsequent semantic processing. QILP-0 constructs
a separate matrix for geometric analysis using standard centering and scaling
operations
\cite{shlens2014tutorial,bro2003centering,jolliffe2016pca}.

Let
\(x_j\in\mathbb R^m\)
denote column \(j\) of \(X_{\mathrm{obs}}\), and let
\(\mathbf 1_m\in\mathbb R^m\)
be the all-ones vector. Define

\begin{equation}
	\bar x_j
	=
	\frac{1}{m}
	\mathbf 1_m^T x_j,
	\label{eq:qilp-column-mean}
\end{equation}

\begin{equation}
	x_{c,j}
	=
	x_j-\bar x_j\mathbf 1_m,
	\qquad
	\widehat x_j
	=
	\frac{x_{c,j}}{\lVert x_{c,j}\rVert_2}.
	\label{eq:qilp-column-preprocessing}
\end{equation}

A column is valid for this normalized geometry when its entries are finite and
its centred norm is non-zero, or non-zero under the declared numerical
tolerance. If
\(J_{\mathrm{valid}}\subseteq\{1,\ldots,p\}\)
is the corresponding index set and
\(p_{\mathrm{valid}}=\lvert J_{\mathrm{valid}}\rvert\),

\begin{equation}
	X_{\mathrm{geom}}
	=
	\left[
	\widehat x_j
	\right]_{j\in J_{\mathrm{valid}}}
	\in
	\mathbb R^{m\times p_{\mathrm{valid}}}.
	\label{eq:qilp-geometric-matrix}
\end{equation}

Centering prevents common offsets from dominating the singular directions,
while unit-norm scaling prevents an observable from receiving greater initial
geometric weight merely because its empirical dispersion is larger.

For
\(j,\ell\in J_{\mathrm{valid}}\),

\begin{equation}
	\widehat x_j^T\widehat x_\ell
	=
	\frac{
		x_{c,j}^Tx_{c,\ell}
	}{
		\lVert x_{c,j}\rVert_2
		\lVert x_{c,\ell}\rVert_2
	}
	=
	\operatorname{corr}(x_j,x_\ell),
	\label{eq:qilp-profile-correlation}
\end{equation}

where
\(\operatorname{corr}\)
denotes empirical Pearson correlation. Consequently,

\begin{equation}
	X_{\mathrm{geom}}^T X_{\mathrm{geom}}
	=
	R,
	\label{eq:qilp-correlation-gram}
\end{equation}

where \(R\) is the correlation matrix of the valid observable profiles.
Because every valid column has unit Euclidean norm,

\begin{equation}
	\lVert X_{\mathrm{geom}}\rVert_F^2
	=
	\sum_{j\in J_{\mathrm{valid}}}
	\lVert\widehat x_j\rVert_2^2
	=
	p_{\mathrm{valid}}.
	\label{eq:qilp-normalized-geometric-mass}
\end{equation}

For a real matrix
\(A\in\mathbb R^{m\times p}\),
let

\begin{equation}
	A
	=
	U\Sigma V^T
	\label{eq:qilp-svd-definition}
\end{equation}

be its thin singular value decomposition. If
\(r=\operatorname{rank}(A)\),
then \(U\) and \(V\) have orthonormal columns and

\begin{equation}
	\Sigma
	=
	\operatorname{diag}
	\left(
	\sigma_1,\ldots,\sigma_r
	\right),
	\qquad
	\sigma_1\geq\cdots\geq\sigma_r>0,
	\label{eq:qilp-svd-dimensions}
\end{equation}

with the standard SVD interpretation
\cite{golub1970svd}. Moreover,

\begin{equation}
	\lVert A\rVert_F^2
	=
	\sum_{\ell=1}^{r}\sigma_\ell^2.
	\label{eq:qilp-svd-geometric-mass}
\end{equation}

The squared singular values therefore distribute the matrix's geometric mass
among orthogonal singular directions.

\paragraph{Intuition.}
After centering and scaling, each observable is represented by the shape of
its response across the same input situations rather than by its absolute
offset or raw amplitude. SVD then asks how many independent patterns are
needed to organize those response profiles. If many observables vary in
closely related ways, a small number of directions can describe much of the
geometry. If substantially more independent directions are required, the
observed representation is more articulated.

\paragraph{Interpretation.}
This geometry is useful to QILP-0 because it provides a target-independent
account of the diversity and structure of observable profiles. Within the
declared dataset and observational reference horizon, numerical rank,
singular-value structure, and later coverage growth can therefore be
interpreted as operational indicators of \emph{observational structural
richness}.
This is
the restricted sense in which the original formulation informally referred to
information or complexity. The quantities are not physical energy,
Hamiltonian expectation values, Shannon information, intrinsic quantum
variance, or a universal measure of circuit complexity.

The construction is closely related to standard PCA--SVD geometry
\cite{shlens2014tutorial,jolliffe2016pca}. PCA has also been applied to
covariance matrices of quantum observables within a fixed state
\cite{mosetti2016quantum}. QILP-0 studies a different object: observable
expectation-value profiles evaluated across multiple semantically identifiable
inputs. In addition, as Condition 7 will make explicit, the latent singular
coordinates are not used as the variables of the later symbolic theory.

\paragraph{Relation to supervised feature selection.}
The role of this geometric stage should not be confused with supervised feature
selection. A decision tree, random forest, mutual-information criterion, or
other target-based importance measure could identify variables that are useful
for predicting the supplied labels, but doing so would make the observational
vocabulary depend on the target that the later declarative theory is intended
to reconstruct. QILP-0 asks a different question at this stage: which
redundancies and numerically resolvable directions are present in the
observable-response matrix itself, before the target is consulted? SVD is used
here for that target-independent geometric purpose. Moreover, its latent
directions are not passed to the symbolic learner: the construction maps the
retained geometry back to original observable columns, as formalized in
Condition~7, so that the declarative vocabulary preserves its quantum
provenance.

\subsection{Condition 5: incremental geometric progress can be measured against a fixed reference}
\label{subsec:condition-fixed-reference}

Conditions 3 and 4 provide, respectively, an ordered sequence of observable
blocks and a geometry in which their empirical structure can be analysed.
QILP-0 next requires that progress through this sequence be measured against a
reference that is fixed independently of the incremental decisions themselves.

For every exact grade
\(k\in\{1,\ldots,K_{\mathrm{ref}}\}\),
let
\(X^{(k)}\in\mathbb R^{m\times p_k}\)
denote the submatrix of
\(X_{\mathrm{geom}}\)
formed by the valid \(p_k\) columns associated with
\(\mathcal O^{(k)}\).
All blocks share the same \(m\) observational rows. The normalized matrix of
the declared reference horizon is
\begin{equation}
	\begin{split}
		X_{\mathrm{ref}}
		&=
		\left[
		X^{(1)}
		\;X^{(2)}
		\;\cdots\;
		X^{(K_{\mathrm{ref}})}
		\right]
		\in
		\mathbb R^{m\times p_{\mathrm{ref}}},
		\\
		p_{\mathrm{ref}}
		&=
		\sum_{k=1}^{K_{\mathrm{ref}}}p_k.
		\label{eq:qilp-global-horizon-matrix}
	\end{split}
\end{equation}

The value \(K_{\mathrm{ref}}\) has already been declared before the fixed
reference is constructed. QILP-0 does not require a particular mechanism for
choosing it. What the construction does require is that every complete grade
included in the declared reference horizon contribute to the same fixed
reference.

A reference-calibration pass therefore obtains the observable responses
required for all complete grades
\(1,\ldots,K_{\mathrm{ref}}\), either by evaluating them through the declared
provider or by retrieving equivalent stored observational data, and
accumulates their contribution to the fixed geometric reference. The
implementation need not keep all observable columns simultaneously resident
in memory.

Only after this reference has been fixed does the incremental traversal begin.
The traversal may terminate at a support
\(K_{\mathrm{end}}<K_{\mathrm{ref}}\)
because a stopping criterion has been reached or because the next complete
grade cannot be propagated through the configured downstream processing.
A grade that belongs to the reference horizon may therefore contribute to
reference calibration without becoming a completed incremental layer. This
does not redefine the reference denominator.

QILP-0 fixes before incremental selection the reference row Gram matrix

\begin{equation}
	G_{\mathrm{ref}}
	=
	X_{\mathrm{ref}}X_{\mathrm{ref}}^T
	\in
	\mathbb R^{m\times m},
	\label{eq:qilp-global-gram}
\end{equation}

and the reference geometric mass

\begin{equation}
	M_{\mathrm{ref}}
	=
	\operatorname{Tr}
	\left(
	G_{\mathrm{ref}}
	\right)
	=
	\lVert X_{\mathrm{ref}}\rVert_F^2.
	\label{eq:qilp-total-geometric-mass}
\end{equation}

Because every valid column is normalized to unit Euclidean norm,
\(M_{\mathrm{ref}}=p_{\mathrm{ref}}\)
under the unweighted normalization defined in Condition 4, up to numerical
roundoff.

The norm-based definition is retained because it extends naturally to
weighted, uncertainty-aware, or provider-dependent variants.

For any
\(Q\in\mathbb R^{m\times r_Q}\)
with orthonormal columns,
\(Q^TQ=I_{r_Q}\),
where \(r_Q\) denotes the number of orthonormal basis vectors in \(Q\),
and therefore the dimension of the subspace represented by \(Q\),
let
\(P_Q=QQ^T\)
be the orthogonal projector onto
\(\operatorname{span}(Q)\).
Here,

\begin{equation}
	\operatorname{span}(Q)
	=
	\{Qa:a\in\mathbb{R}^{r_Q}\}
\end{equation}

denotes the column space generated by those basis vectors.
In the incremental construction introduced below, \(Q=Q_k\), so \(r_Q\)
corresponds to the number of independent geometric directions accumulated
up to grade \(k\).

The geometric mass of the reference matrix represented by this subspace is

\begin{equation}
	M_Q
	=
	\lVert Q^T X_{\mathrm{ref}}\rVert_F^2
	=
	\operatorname{Tr}
	\left(
	Q^T G_{\mathrm{ref}} Q
	\right).
	\label{eq:qilp-projected-geometric-mass}
\end{equation}

The complementary geometric mass, corresponding to the component of
\(X_{\mathrm{ref}}\)
orthogonal to
\(\operatorname{span}(Q)\),
is

\begin{equation}
	M_Q^\perp
	=
	\lVert
	(I_m-QQ^T)X_{\mathrm{ref}}
	\rVert_F^2.
	\label{eq:qilp-residual-geometric-mass}
\end{equation}

Since the projected and complementary components are orthogonal in the
Frobenius inner product, their geometric masses add exactly:

\begin{equation}
	M_{\mathrm{ref}}
	=
	M_Q
	+
	M_Q^\perp.
	\label{eq:qilp-geometric-mass-decomposition}
\end{equation}

The Frobenius norm, orthogonal projection, and SVD provide the standard
matrix-geometric framework underlying this construction
\cite{halko2011finding}.

\paragraph{Coverage-reference independence.}
The objects
\(X_{\mathrm{ref}}\),
\(G_{\mathrm{ref}}\),
and
\(M_{\mathrm{ref}}\)
are fixed before incremental selection begins. They are independent of the
supervised target, the basis eventually retained, the original observables
eventually passed to the symbolic layer, and the later stopping decision.
Hence, neither a coverage threshold nor a downstream resource endpoint can
alter the reference against which coverage is measured.

The traversal then asks what geometrically new behaviour appears when the
next exact-grade block is introduced. Before processing grade \(k\), let

\begin{equation}
	Q_{k-1}
	\in
	\mathbb R^{m\times r_{k-1}},
	\qquad
	Q_{k-1}^TQ_{k-1}=I_{r_{k-1}},
	\label{eq:qilp-previous-orthonormal-basis}
\end{equation}

span the left singular subspace accumulated from previous blocks. For
\(k=1\), \(Q_0\) is empty and
\(Q_0Q_0^T=0\).

The incoming block decomposes as

\begin{equation}
	X^{(k)}
	=
	Q_{k-1}Q_{k-1}^T X^{(k)}
	+
	\widetilde X^{(k)},
	\label{eq:qilp-block-parallel-residual-decomposition}
\end{equation}

where

\begin{equation}
	\widetilde X^{(k)}
	=
	\left(
	I_m-Q_{k-1}Q_{k-1}^T
	\right)
	X^{(k)}.
	\label{eq:qilp-incremental-residual}
\end{equation}

The first term is already represented by the accumulated subspace. The second
is the local geometric novelty. 
This residual characterizes geometric novelty relative to the previously
accumulated subspace. By itself, it does not measure reference-relative
coverage gain, predictive relevance, information-theoretic content, or
symbolic importance.

This is the standard decomposition underlying
incremental SVD updates
\cite{brand2002incremental,brand2006fast}, and

\begin{equation}
	Q_{k-1}^T
	\widetilde X^{(k)}
	=
	0.
	\label{eq:qilp-residual-orthogonality}
\end{equation}

Let

\begin{equation}
	\widetilde X^{(k)}
	=
	U^{(k)}
	\Sigma^{(k)}
	{V^{(k)}}^T
	\label{eq:qilp-residual-svd}
\end{equation}

be its thin SVD. QILP-0 uses a declared component-retention parameter
\(\tau_{\mathrm{dir}}\in(0,1]\)
to determine how much of the resolvable local residual geometry is retained
in the new subspace. In a direct SVD realization, the smallest leading set of
singular directions whose squared singular values reach the declared
fraction of the local residual geometric mass is retained, subject to the
configured numerical-rank tolerance. Equivalent residual-Gram or structured
implementations may realize the same retained subspace without explicitly
materializing the full SVD.

The component-retention parameter controls geometric directions only. It is
distinct from the original-observable selection parameter introduced in
Condition~7, which determines how much of the association between original
columns and the retained subspace must be preserved.

If
\(\mathcal R_k\)
is their index set,

\begin{equation}
	Q_{\mathrm{add}}^{(k)}
	=
	\left[
	u_\ell^{(k)}
	\right]_{\ell\in\mathcal R_k},
	\label{eq:qilp-added-directions}
\end{equation}
and

\begin{equation}
	Q_k
	=
	\left[
	Q_{k-1}
	\quad
	Q_{\mathrm{add}}^{(k)}
	\right].
	\label{eq:qilp-incremental-basis-update}
\end{equation}
so that
\(Q_{\mathrm{add}}^{(k)}
\in\mathbb R^{m\times|\mathcal R_k|}\).
Consequently,

\[
r_k
=
r_{k-1}
+
|\mathcal R_k|
\]

is the number of retained independent geometric directions accumulated after
processing grade \(k\). We refer to \(r_k\) as the
\emph{accumulated retained geometric dimension}.

After any required numerical reorthogonalization,
\(Q_k^TQ_k=I_{r_k}\).

The construction does not require a particular incremental-SVD
implementation. An explicit residual SVD, a structured thin-SVD update, or
an equivalent numerically stable implementation may be used, provided that
the retained subspace, orthogonality conditions, and provenance remain
consistent with the definitions above.

\paragraph{Intuition.}
Before asking what support \(k\) contributes, QILP-0 removes everything in
that block that can already be represented using directions discovered at
lower supports. The residual therefore asks: \emph{what is genuinely new in
	the observable geometry at this level?} The answer enlarges the accumulated
basis, but progress is always evaluated against the same declared reference
horizon.

Four related quantities make the distinction between local novelty and
reference-relative progress explicit.
The geometric mass of the incoming block is

\begin{equation}
	M_{\mathrm{block}}(k)
	=
	\lVert X^{(k)}\rVert_F^2,
	\label{eq:qilp-block-geometric-mass}
\end{equation}

and its local residual mass is

\begin{equation}
	M_{\mathrm{res}}^{\mathrm{local}}(k)
	=
	\lVert\widetilde X^{(k)}\rVert_F^2.
	\label{eq:qilp-local-residual-mass}
\end{equation}

After processing grade \(k\), the mass of the declared reference horizon
represented by the accumulated basis is

\begin{equation}
	M_{\leq k}^{\mathrm{ref}}
	=
	\operatorname{Tr}
	\left(
	Q_k^T G_{\mathrm{ref}} Q_k
	\right)
	=
	\lVert
	Q_k^T X_{\mathrm{ref}}
	\rVert_F^2.
	\label{eq:qilp-global-represented-mass}
\end{equation}

The corresponding increment in represented reference mass is

\begin{equation}
	\Delta M_{\mathrm{ref}}(k)
	=
	M_{\leq k}^{\mathrm{ref}}
	-
	M_{\leq k-1}^{\mathrm{ref}}.
	\label{eq:qilp-global-geometric-increment}
\end{equation}

In general,

\begin{equation}
	\Delta M_{\mathrm{ref}}(k)
	\neq
	M_{\mathrm{res}}^{\mathrm{local}}(k).
	\label{eq:qilp-local-global-mass-distinction}
\end{equation}

The local residual is evaluated only on the incoming block and may contain
directions that are not retained. By contrast,
\(\Delta M_{\mathrm{ref}}(k)\)
is evaluated against the complete fixed reference and may also account for
structure in other blocks of that reference that is represented by the newly
retained directions.

\paragraph{Intuition.}
A book provides a useful mental model. Suppose exact-support blocks are
chapters read in order.
\(M_{\mathrm{block}}(k)\)
is how much normalized observational structure chapter \(k\) contains;
\(M_{\mathrm{res}}^{\mathrm{local}}(k)\)
is the part of that chapter that the previous summary could not already
represent;
\(M_{\leq k}^{\mathrm{ref}}\)
is how much of the complete declared reference can be represented by the
summary built after reading through chapter \(k\); and
\(\Delta M_{\mathrm{ref}}(k)\)
is the improvement, measured against that same fixed reference, attributable
to the new directions retained at this step. The reference book is fixed
before the incremental accounting begins. This is why local novelty cannot
simply be accumulated and interpreted as reference-relative coverage.

All these quantities refer to normalized observational geometry within the
declared reference horizon. Their role is to operationalize the amount of
observable structure represented at each stage. They are not exact values of
physical energy, Shannon information, computational complexity, or
target-dependent relevance.

\subsection{Condition 6: the traversal is monotone and admits explicit stopping conditions}
\label{subsec:condition-stopping}

QILP-0 requires formal guarantees that the incremental traversal has
measurable progress and an explicit endpoint.

Let
\(X_{\mathrm{ref}}\in\mathbb R^{m\times p_{\mathrm{ref}}}\)
be the fixed normalized reference matrix of Condition 5 and assume
\begin{equation}
	M_{\mathrm{ref}}
	=
	\lVert X_{\mathrm{ref}}\rVert_F^2
	>
	0.
	\label{eq:qilp-positive-total-geometric-mass}
\end{equation}

After exact-support blocks through grade \(k\) have been processed, let
\(Q_k\)
be the accumulated orthonormal basis. QILP-0 defines
\emph{reference-relative observational coverage} as
\begin{equation}
	C_{\mathrm{ref}}(k)
	=
	\frac{
		M_{\leq k}^{\mathrm{ref}}
	}{
		M_{\mathrm{ref}}
	}
	=
	\frac{
		\lVert Q_k^TX_{\mathrm{ref}}\rVert_F^2
	}{
		\lVert X_{\mathrm{ref}}\rVert_F^2
	}.
	\label{eq:qilp-global-coverage}
\end{equation}

Thus,
\(C_{\mathrm{ref}}(k)\)
is the fraction of normalized geometric mass of the declared reference horizon
represented by
\(\operatorname{span}(Q_k)\).
When
\(K_{\mathrm{ref}}<K_{\mathrm{phys}}\),
this quantity is not an estimate of the fraction of the unrestricted Pauli
geometry represented beyond \(K_{\mathrm{ref}}\).

This parallels proportions of total variance in PCA
\cite{shlens2014tutorial,jolliffe2016pca},
while the projection geometry follows the standard framework summarized in
\cite{halko2011finding}.

Since

\begin{equation}
	X_{\mathrm{ref}}
	=
	Q_kQ_k^T X_{\mathrm{ref}}
	+
	(I_m-Q_kQ_k^T)X_{\mathrm{ref}}
	\label{eq:qilp-global-orthogonal-decomposition}
\end{equation}

is an orthogonal decomposition,

\begin{equation}
	1-C_{\mathrm{ref}}(k)
	=
	\frac{
		\lVert
		(I_m-Q_kQ_k^T)X_{\mathrm{ref}}
		\rVert_F^2
	}{
		\lVert X_{\mathrm{ref}}\rVert_F^2
	}.
	\label{eq:qilp-normalized-global-residual}
\end{equation}

The accumulated bases are nested,

\begin{equation}
	Q_k
	=
	\left[
	Q_{k-1}
	\quad
	Q_{\mathrm{add}}^{(k)}
	\right],
	\qquad
	Q_{k-1}^TQ_{\mathrm{add}}^{(k)}=0,
	\label{eq:qilp-nested-basis}
\end{equation}

and therefore

\begin{equation}
	C_{\mathrm{ref}}(k)
	=
	C_{\mathrm{ref}}(k-1)
	+
	\frac{
		\left\lVert
		{Q_{\mathrm{add}}^{(k)}}^T X_{\mathrm{ref}}
		\right\rVert_F^2
	}{
		M_{\mathrm{ref}}
	}.
	\label{eq:qilp-coverage-monotone-update}
\end{equation}

Hence,

\begin{equation}
	0
	\leq
	C_{\mathrm{ref}}(k-1)
	\leq
	C_{\mathrm{ref}}(k)
	\leq
	1,
	\label{eq:qilp-coverage-bounds}
\end{equation}

and

\begin{equation}
	\Delta C_{\mathrm{ref}}(k)
	=
	C_{\mathrm{ref}}(k)-C_{\mathrm{ref}}(k-1)
	=
	\frac{
		\Delta M_{\mathrm{ref}}(k)
	}{
		M_{\mathrm{ref}}
	}
	\geq0.
	\label{eq:qilp-coverage-increment}
\end{equation}

\paragraph{Interpretation.}
Coverage answers a restricted and operational question: how much of the
observable-response geometry that QILP-0 committed to inspect is already
represented by the directions discovered up to support \(k\)? A high coverage
means that little additional geometric mass remains unresolved inside that
declared reference horizon. It does not imply that the unrestricted quantum
state, every possible observable, or all physical information has been
reconstructed.

The stopping policy is evaluated only after a complete exact-grade block. Let
\(\tau_{\mathrm{cov}}\in(0,1]\)
be the coverage target,
\(\tau_{\mathrm{plateau}}>0\)
the small-increment threshold, and
\(s_{\mathrm{plateau}}\in\mathbb N\)
the number of consecutive small increments required to confirm a plateau. The
reference configuration used in the experiments is

\begin{equation}
	\tau_{\mathrm{cov}}=0.99,
	\qquad
	\tau_{\mathrm{plateau}}=0.005,
	\qquad
	s_{\mathrm{plateau}}=2.
	\label{eq:qilp-reference-thresholds}
\end{equation}

The traversal reaches an endpoint under one of the following conditions:

\begin{enumerate}
	\item \textbf{Reference coverage reached:}
	\(C_{\mathrm{ref}}(k)\geq\tau_{\mathrm{cov}}\).
	
	\item \textbf{Plateau confirmed, when enabled as a stopping condition:}
	the last
	\(s_{\mathrm{plateau}}\)
	complete blocks satisfy
	\(\Delta C_{\mathrm{ref}}(t)<\tau_{\mathrm{plateau}}\).
	
	\item \textbf{Downstream resource limit:}
	the next complete grade
	\(k+1\leq K_{\mathrm{ref}}\)
	cannot be propagated through the configured incremental/downstream
	processing stages. No partial subset of that grade is substituted for the
	declared block.
	
	\item \textbf{Reference horizon exhausted:}
	\(k=K_{\mathrm{ref}}\).
\end{enumerate}

The support of the last completed incremental layer is recorded as
\(K_{\mathrm{end}}\).
Hence,
\[
1\leq K_{\mathrm{end}}\leq K_{\mathrm{ref}}\leq K_{\mathrm{phys}}.
\]

The thresholds, active conditions, and final stop reason are recorded as
execution metadata.
If
\(M_{\mathrm{ref}}=0\),
coverage is undefined and the execution terminates with a
degenerate-geometry status. Numerical implementations also monitor
\(\lVert Q_k^TQ_k-I_{r_k}\rVert\)
and diagnose any decrease in coverage beyond the declared numerical
tolerance.

\paragraph{Intuition.}
The process can be imagined as progressively improving a summary while the
object to be summarized remains fixed. Coverage tells us when the summary
already represents almost all the geometric structure that was declared in
scope. A plateau warns that successive increases in structural grade are
adding very little. 
A resource endpoint says something different: it does not
claim convergence, only that the next complete observational layer cannot be
admitted and propagated through the configured downstream stages under the
declared budget.

The plateau criterion is particularly useful as diagnostic information in an
interactive execution. A domain expert may decide whether a warning justifies
stopping or whether exploration should continue. In autonomous executions, the
same condition can be configured as a warning rather than an unconditional
stop; its occurrence is nevertheless stored in the provenance.

The fixed reference may be accumulated through a reference-calibration pass
without requiring all exact-support blocks to be simultaneously resident in
memory. Reference construction and downstream block processing are distinct
resource events. The former must obtain the contribution of every complete
grade in the declared reference horizon; the latter may additionally require
residualization, factorization, original-column selection, downstream
materialization, discretization, and declarative processing.

Consequently, a grade can contribute to
\(G_{\mathrm{ref}}\)
during reference calibration and still fail to become a completed incremental
layer. Such an endpoint changes
\(K_{\mathrm{end}}\)
but not
\(K_{\mathrm{ref}}\),
\(G_{\mathrm{ref}}\), or
\(M_{\mathrm{ref}}\).

\subsection{Condition 7: latent geometry can be mapped back to the original observable vocabulary}
\label{subsec:condition-original-vocabulary}

The previous conditions establish a latent geometric description of the
observational matrix. For QILP-0, however, a declarative explanation is useful
only if the variables that reach the symbolic layer remain semantically
identifiable. The final condition of the continuous layer is therefore that
the retained geometry can be related back to the original observable
vocabulary without substituting latent components for physical observables.

For exact grade \(k\), recall

\begin{equation}
	\widetilde X^{(k)}
	=
	U^{(k)}
	\Sigma^{(k)}
	{V^{(k)}}^T.
	\label{eq:qilp-residual-svd-selection}
\end{equation}

Let
\[
r_k^{\mathrm{res}}
=
\operatorname{rank}
\left(
\widetilde X^{(k)}
\right).
\]

Writing

\begin{equation}
	\widetilde X^{(k)}
	=
	\left[
	\widetilde x_1^{(k)}
	\;\cdots\;
	\widetilde x_{p_k}^{(k)}
	\right],
\end{equation}

the \(j\)-th residualized observable profile has the singular-triplet
expansion

\begin{equation}
	\widetilde x_j^{(k)}
	=
	\sum_{\ell=1}^{r_k^{\mathrm{res}}}
	\sigma_\ell^{(k)}
	v_{j\ell}^{(k)}
	u_\ell^{(k)}.
	\label{eq:qilp-column-triplets}
\end{equation}

Equivalently,

\begin{equation}
	{u_\ell^{(k)}}^T
	\widetilde x_j^{(k)}
	=
	\sigma_\ell^{(k)}
	v_{j\ell}^{(k)},
	\label{eq:qilp-singular-association-coordinate}
\end{equation}

so
\(\sigma_\ell^{(k)}v_{j\ell}^{(k)}\)
is the coordinate of residualized observable profile \(j\) along left singular
direction \(\ell\).

If
\(\mathcal R_k\)
is the set of retained singular directions, define

\begin{equation}
	a_j^{(k)}
	=
	\left(
	\sigma_\ell^{(k)}
	v_{j\ell}^{(k)}
	\right)_{\ell\in\mathcal R_k}
	\in
	\mathbb R^{|\mathcal R_k|},
	\label{eq:qilp-observable-association-vector}
\end{equation}

with magnitude

\begin{equation}
	\alpha_j^{(k)}
	=
	\left\|a_j^{(k)}\right\|_2
	=
	\sqrt{
		\sum_{\ell\in\mathcal R_k}
		\left(
		\sigma_\ell^{(k)}
		v_{j\ell}^{(k)}
		\right)^2
	}.
	\label{eq:qilp-observable-association-magnitude}
\end{equation}

These quantities describe the association between original residualized
observable profiles and the newly retained geometric directions. The
singular-triplet interpretation follows standard SVD and incremental-SVD
geometry
\cite{brand2002incremental,brand2006fast},
while the distinction between latent directions and original variables is
consistent with the usual PCA--SVD interpretation
\cite{shlens2014tutorial,jolliffe2016pca}.

The retained directions define a geometric subspace, but QILP-0 does not use
their latent coordinates as variables of the symbolic dataset. Instead, it
maps that subspace back to the valid original observable columns of the same
exact-grade block.

Let
\(\mathcal V_k\)
denote the ordered index set of valid, non-constant original observable
columns of grade \(k\), and let
\(Q_{\mathrm{add}}^{(k)}\)
be an orthonormal basis for the new subspace retained at that grade. For
\(j\in\mathcal V_k\), define the original-observable association score

\begin{equation}
	s_j^{(k)}
	=
	\left\|
	{Q_{\mathrm{add}}^{(k)}}^T
	x_j^{(k)}
	\right\|_2^2,
	\label{eq:qilp-original-observable-score}
\end{equation}

where
\(x_j^{(k)}\)
is the normalized geometric profile of the original observable before local
residualization. Since
\(Q_{\mathrm{add}}^{(k)}\)
is orthogonal to the previously accumulated subspace,

\begin{equation}
	s_j^{(k)}
	=
	\left\|
	{Q_{\mathrm{add}}^{(k)}}^T
	\widetilde x_j^{(k)}
	\right\|_2^2.
	\label{eq:qilp-original-observable-score-residual}
\end{equation}

In an explicit residual-SVD realization in which
\(Q_{\mathrm{add}}^{(k)}\)
is formed by the retained left singular directions, this score is equivalently

\begin{equation}
	s_j^{(k)}
	=
	\sum_{\ell\in\mathcal R_k}
	\left(
	\sigma_\ell^{(k)}
	v_{j\ell}^{(k)}
	\right)^2
	=
	\left(
	\alpha_j^{(k)}
	\right)^2.
	\label{eq:qilp-original-observable-score-svd}
\end{equation}

Thus, the score depends on the retained geometric subspace rather than on a
particular choice of orthonormal basis within that subspace.

Let

\begin{equation}
	S_k
	=
	\sum_{j\in\mathcal V_k}
	s_j^{(k)}
	\label{eq:qilp-total-original-observable-score}
\end{equation}

be the total original-column association score for grade \(k\). If
\(S_k=0\), the retained subspace has no measurable association with any valid
original column under the declared numerical policy and no new symbolic
observable is added from that grade.

Otherwise, let
\(\pi_k\)
be the deterministic permutation of
\(\mathcal V_k\)
that orders columns by non-increasing score,

\begin{equation}
	s_{\pi_k(1)}^{(k)}
	\geq
	s_{\pi_k(2)}^{(k)}
	\geq
	\cdots,
	\label{eq:qilp-original-observable-score-order}
\end{equation}

with exact score ties resolved by a stable target-independent structural
ordinal attached to every observable.

For a declared column-association coverage parameter
\(\tau_{\mathrm{col}}\in(0,1]\),
define

\begin{equation}
	r_k^\star
	=
	\min
	\left\{
	r:
	\frac{
		\sum_{t=1}^{r}
		s_{\pi_k(t)}^{(k)}
	}{
		S_k
	}
	\geq
	\tau_{\mathrm{col}}
	\right\}.
	\label{eq:qilp-original-observable-cutoff}
\end{equation}

The original observables selected from exact grade \(k\) are then

\begin{equation}
	\mathcal J_k
	=
	\left\{
	\pi_k(1),
	\ldots,
	\pi_k(r_k^\star)
	\right\}.
	\label{eq:qilp-original-observable-selected-set}
\end{equation}

This criterion selects the smallest deterministic score-ranked prefix whose
cumulative association with the retained new geometry reaches the declared
column threshold. It is not a supervised ranking, a fixed top-\(N\) rule, or
a replacement of original observables by latent components.

The selected sets are accumulated across processed grades. At the symbolic
boundary, their source columns are placed in a stable target-independent
structural order, which need not coincide with the score ranking used to
decide membership in
\(\mathcal J_k\).
Let
\(\mathcal J_{\leq k}\)
denote that ordered accumulation. The corresponding symbolic vocabulary is

\begin{equation}
	\mathcal O_{\mathrm{sym}}^{(\leq k)}
	=
	\left\{
	O_j:
	j\in\mathcal J_{\leq k}
	\right\}.
	\label{eq:qilp-symbolic-vocabulary}
\end{equation}

If
\(x_j\in\mathbb R^m\)
is the original continuous observational profile associated with observable
\(O_j\), the dataset delivered to discretization and declarative induction is

\begin{equation}
	X_{\mathrm{sym}}^{(\leq k)}
	=
	\left[
	x_j
	\right]_{j\in\mathcal J_{\leq k}}.
	\label{eq:qilp-symbolic-continuous-dataset}
\end{equation}

Consequently, the number of retained geometric directions need not equal the
number of original variables passed to the symbolic stage. Several original
observables may contribute to the same retained direction, and one observable
may contribute to several directions.

\paragraph{Intuition.}
SVD tells QILP-0 which new independent directions organize the observable
geometry, but those directions are not the vocabulary in which the final
theory should be written. A singular direction may mix many physically
identifiable observables. QILP-0 therefore returns to the original columns and
asks how strongly each of them participates in the newly retained geometric
subspace.

The first threshold,
\(\tau_{\mathrm{dir}}\),
controls how much of the local residual geometry is retained as numerical
directions. The second,
\(\tau_{\mathrm{col}}\),
controls how much of the resulting original-column association is preserved
when returning from that latent geometry to the observable vocabulary.
Accordingly, QILP-0 does not keep an observable because it predicts the target,
nor does it require every column with an arbitrarily small non-zero projection
to survive. It keeps the deterministic score-ranked set of original
observables required to preserve the declared fraction of their association
with the retained new geometry.

Thus, geometric reduction is used to decide which original observables must
be preserved, while the latent singular coordinates themselves never become
variables of the declarative theory.

This design supports the provenance chain

\begin{equation}
	\begin{split}
		\text{rule literal}
		&\longrightarrow
		\text{discrete state}
		\longrightarrow
		\text{interval or native value}
		\\
		&\longrightarrow
		\text{original observable}
		\longrightarrow
		\text{Pauli word} \\
		& \longrightarrow
		\text{support and qubits}.
	\end{split}
	\label{eq:qilp-traceability-chain}
\end{equation}

The current QILP-0 instantiation uses selected original observables as symbolic
atoms. For every admitted observable, the selection metadata record its source
identifier, exact grade, qubit support, structural ordinal, geometric
association score, and original provider provenance. The active geometric
selection thresholds are recorded separately with the corresponding
support/run configuration.

Future variants may use communities, effective operators, regions, or other
domain-defined semantic entities, but such alternatives require an explicit
translation layer and are outside the present construction.

All operations in Conditions 1--7 are target-independent. The target has not
been used to generate observables, define the structural grading, normalize
profiles, build the fixed observational reference, perform SVD, select
original observable columns, or determine the traversal endpoint. This
target-independence requirement also extends to the discretization of
Condition~8. Once the analysed quantum producer is fixed, QILP-0 first
accesses the supervised target in the twin-admissibility audit of
Condition~9; declarative induction follows only after that audit succeeds.

\section{Declarative reasoning layer: from continuous observations to a finite relation}
\label{sec:declarative-layer}

QILP-0 is organized into two main methodological layers. Conditions 1--7
formalized the continuous processing layer: quantum-circuit behaviour was
represented numerically, explored through a declared observable family,
organized through a target-independent geometry, while preserving a mapping
back to original, semantically traceable observable columns and, through their
provenance, to the qubits on which those observables act.

The declarative reasoning layer begins from that continuous observational
vocabulary. Its first task is not yet to induce rules, but to construct a
finite symbolic representation without using the supervised target to shape
the representation itself. Only after this representation has been fixed does
QILP-0 audit whether it is compatible with an exact deterministic
input--target relation on the analysed dataset.

Accordingly, this section states two additional methodological conditions.
Condition 8 specifies when a column-wise discretization is admissible for
QILP-0. Condition 9 checks, after discretization, whether the resulting finite
representation is twin-admissible for the observed target. Once both
conditions are satisfied, the finite relation presented to the inductive logic
engine can be defined, followed by the contractual definition of the
observational declarative twin that QILP-0 will construct in
Section~\ref{sec:constructive-realization}.

\subsection{Condition 8: each observable admits a target-independent finite discretization at a defensible resolution}
\label{subsec:condition-admissible-discretization}

The declarative stages of QILP-0 require finite-valued variables, whereas the
continuous observational layer preserves the numerical responses of the
selected observables. The transition between both layers therefore raises a
methodological question: for each observable, which finite symbolic
resolution can be justified by the values actually reported and by the
available information about their numerical or observational resolution?

This question is also motivated by the operational way in which quantum
observations are made available to software. A backend never exposes an
abstract real number with unlimited accessible precision: it returns a finite
numerical representation of the requested observable value. Depending on the
evaluation mechanism, the effective resolution may additionally be constrained
by finite sampling, numerical approximation, provider-specific tolerances, or
explicit uncertainty estimates. These mechanisms do not themselves define the
symbolic states required by QILP-0, but they make resolution an intrinsic part
of the observational contract.

QILP-0 therefore treats discretization as a controlled change of
representation. For each observable, it asks how many symbolic distinctions
can be supported by its reported profile and by the applicable resolution
information, rather than imposing an arbitrary common precision on all
observables.

QILP-0 addresses this question without attempting to recover an assumed set
of ``true'' discrete states underlying a continuous observable. Instead, it
constructs a finite symbolic domain whose granularity is constrained by the
empirical resolution of the observed profile and, when available, by
uncertainty or effective-resolution information supplied by the
backend or observable provider. The objective is to preserve the semantics of
the observable while avoiding symbolic distinctions finer than the available
observations can defend.

Let
\[
\mathcal S
=
\{O_1,\ldots,O_d\}
\]
denote the ordered set of original observables delivered by
Condition~7, and let

\begin{equation}
	X_{\mathrm{sym}}
	=
	\left[
	x_1
	\;\cdots\;
	x_d
	\right]
	\in
	\mathbb R^{m\times d}
	\label{eq:qilp-symbolic-input-matrix}
\end{equation}

be their continuous observational matrix over the same \(m\) analysed rows.
QILP-0 associates each selected observable \(O_j\) with a finite
discretization map
\(\Delta_j\).
The complete discretizer is therefore

\begin{equation}
	\Delta
	=
	\left(
	\Delta_1,\ldots,\Delta_d
	\right).
	\label{eq:qilp-columnwise-discretizer}
\end{equation}

The maps are fitted without using the supervised target. In the terminology
of classical discretization, the resulting construction is global, marginal,
static by column, and unsupervised
\cite{dougherty1995discretization}: each observable is partitioned using its
complete analysed profile, and the possible target does not participate in
determining its symbolic states.

\cite{dougherty1995discretization}. Here, \emph{global} means that the cuts
for one column are fitted from all analysed values of that column;
\emph{marginal} means that columns are discretized independently; and
\emph{unsupervised} means that the target does not participate in the fitting
process.

Each selected column is first classified as
\emph{constant}, \emph{native discrete}, or \emph{continuous}. Declared
provider metadata have priority. When those metadata are absent or
inconclusive, an empirical audit groups numerical values that are
indistinguishable under a declared column-wise tolerance.

Constant columns receive their unique state and their degeneracy is recorded.
If a column is declared or audited as an exact native discrete variable, its
distinguishable observed values are encoded bijectively and are not
requantized. Only genuinely continuous columns enter the bin-cardinality and
cut-construction procedure described below.

\paragraph{Intuition.}
A numerical observable profile already reaches QILP-0 through a finite
operational representation. Depending on the evaluation mechanism, its
meaningful resolution may be limited only by numerical representation, or
more strongly by sampling, approximation, or declared uncertainty. The
symbolic layer should respect those limits. A profile that is already
natively discrete should therefore preserve its existing states, whereas a
continuous profile should be partitioned only as finely as its empirical and
declared observational resolution can justify.

For a genuinely continuous column
\(x_j\in\mathbb R^m\), define

\begin{equation}
	\begin{split}
		R_j
		&=
		\max_{1\leq i\leq m}x_{ij}
		-
		\min_{1\leq i\leq m}x_{ij},
		\\
		\operatorname{IQR}(x_j)
		&=
		Q_{0.75}(x_j)-Q_{0.25}(x_j),
	\end{split}
	\label{eq:qilp-discretization-range-iqr}
\end{equation}

where
\(Q_\gamma(x_j)\)
denotes the empirical quantile of order
\(\gamma\).
The Freedman--Diaconis candidate width is

\begin{equation}
	h_{\mathrm{FD},j}
	=
	2\,\operatorname{IQR}(x_j)m^{-1/3},
	\label{eq:qilp-fd-width}
\end{equation}

and, whenever
\(h_{\mathrm{FD},j}>0\),
the corresponding candidate number of bins is

\begin{equation}
	q_{\mathrm{FD},j}
	=
	\left\lceil
	\frac{R_j}{h_{\mathrm{FD},j}}
	\right\rceil.
	\label{eq:qilp-fd-bins}
\end{equation}

The Freedman--Diaconis scale provides a data-dependent compromise between
bin width and sampling variability
\cite{freedman1981histogram}.

Let
\(q_{\mathrm{distinct},j}\)
denote the number of empirically resolvable value levels in column
\(x_j\).
Two numerical values are treated as indistinguishable when their separation
does not exceed a declared tolerance
\(\delta_j\).

Let
\(v_{j,(1)}\leq\cdots\leq v_{j,(m_j)}\)
be the ordered finite values of column
\(x_j\).
QILP-0 applies a deterministic ordered tolerance-grouping rule. The first
value initializes the first resolvable level and acts as its representative.
Each subsequent value is assigned to the current level when its distance from
the current representative does not exceed
\(\delta_j\); otherwise, it initializes a new level and becomes its
representative. The resulting number of groups defines
\(q_{\mathrm{distinct},j}\).

The distinguishability tolerance is

\begin{equation}
	\delta_j
	=
	\begin{cases}
		\tau_{\mathrm{native}},
		&
		\text{no backend uncert. available},
		\\[1mm]
		\max\!\left(
		\tau_{\mathrm{native}},
		\lambda_{\mathrm{noise}}\eta_j
		\right),
		&
		\text{otherwise},
	\end{cases}
	\label{eq:qilp-distinguishability-tolerance}
\end{equation}

where
\(\eta_j\geq0\)
is the backend-provided uncertainty or effective resolution when available,
\(\lambda_{\mathrm{noise}}>0\)
is a declared safety multiplier, and
\(\tau_{\mathrm{native}}\)
is the configured numerical fallback tolerance.

In the reference validation configuration,
\(\tau_{\mathrm{native}}=10^{-10}\)
is used as an operational numerical floor. It is not interpreted as a
backend-independent physical accuracy guarantee. Its value and provenance are
recorded and may be replaced by backend- or provider-specific information
when available.

\paragraph{Provider-aware refinement.}
The scalar
\(\tau_{\mathrm{native}}\)
used in the current reference implementation is a fallback rather than a
universal property of observable evaluation. The same observational contract
admits a more informative column-specific numerical floor when suitable
metadata are available:

\begin{equation}
	\tau_{\mathrm{native},j}
	=
	\max
	\left\{
	\tau_{\mathrm{repr},j},
	\tau_{\mathrm{backend},j},
	\tau_{\mathrm{provider},j},
	\tau_{\mathrm{cal},j}
	\right\},
	\label{eq:qilp-native-tolerance-decomposition}
\end{equation}

where the available terms denote, respectively, a numerical-representation
resolution, a declared backend resolution threshold, an observable-provider
approximation bound, and an optional reproducible empirical calibration
bound. Terms for which no applicable information is available are omitted.

Equation~\eqref{eq:qilp-native-tolerance-decomposition} specifies an
extensible refinement of the observational-resolution contract; it does not
claim that all four sources are available in the current reference
implementation or in every backend. When an applicable
\(\tau_{\mathrm{native},j}\)
is available, it can replace the scalar fallback
\(\tau_{\mathrm{native}}\)
in the distinguishability criterion of
Equation~\eqref{eq:qilp-distinguishability-tolerance}.
Otherwise, the configured fallback is retained.

This construction follows the standard metrological principle that
differences below the effective measurement resolution should not be promoted
to distinct empirical states
\cite{jcgm2008gum,nistResolution}. When expectation values are estimated from
finite quantum measurements,
\(\eta_j\)
may also encode the corresponding finite-sampling uncertainty
\cite{crawford2021finiteSampling}.

When the backend provides an uncertainty or effective resolution estimate
\(\eta_j>0\), QILP-0 additionally defines

\begin{equation}
	q_{\mathrm{SNR},j}
	=
	\max
	\left(
	1,
	\left\lfloor
	\frac{R_j}
	{\lambda_{\mathrm{noise}}\eta_j}
	\right\rfloor
	\right).
	\label{eq:qilp-snr-bins}
\end{equation}
Here,
\(\eta_j\)
is the uncertainty or effective-resolution estimate associated with observable
\(j\), and
\(\lambda_{\mathrm{noise}}>0\)
is a declared dimensionless safety multiplier applied to that estimate before
it is used as a distinguishability scale. Its configured value and provenance
are recorded as part of the discretization metadata.

This bound prevents the symbolic vocabulary from requesting distinctions at a
scale finer than the declared observational resolution.

For every continuous column, let
\(\mathcal B_j\)
be the set of available informative bounds,

\begin{equation}
	\mathcal B_j
	=
	\{
	q_{\mathrm{FD},j},
	q_{\mathrm{distinct},j}
	\}
	\cup
	\left\{
	q_{\mathrm{SNR},j}
	\;\middle|\;
	\eta_j>0
	\right\},
	\label{eq:qilp-available-bin-bounds}
\end{equation}

after removing undefined or non-informative candidates. The requested
cardinality is

\begin{equation}
	q_{\mathrm{requested},j}
	=
	\max
	\left(
	2,
	\min\mathcal B_j
	\right).
	\label{eq:qilp-requested-cardinality}
\end{equation}

Thus, the current
policy retains the most restrictive informative column-specific bound,
subject to the minimum cardinality required for a column that remains
classified as genuinely continuous.

A separate sample-sufficiency quantity is retained as a diagnostic warning.
Under the current policy it does not silently reduce
\(q_{\mathrm{requested},j}\)
for an otherwise variable column.

Once
\(q_{\mathrm{requested},j}\)
has been fixed, QILP-0 uses target-independent quantile boundaries by default.
For

\[
r=1,\ldots,q_{\mathrm{requested},j}-1,
\]

the provisional internal cut is

\begin{equation}
	c_{j,r}
	=
	Q_{r/q_{\mathrm{requested},j}}(x_j).
	\label{eq:qilp-quantile-cuts}
\end{equation}

Here,
\(Q_\gamma\)
denotes the empirical quantile function
.

Coincident or numerically indistinguishable cuts are collapsed under the
active tolerance. The effective number of states is therefore

\begin{equation}
	\begin{split}
	q_{\mathrm{effective},j}
	= \\
	1
	+
	\#\{
	\text{distinct retained internal cuts for column }j
	\},
	\label{eq:qilp-effective-cardinality}
	\end{split}
\end{equation}

and may be smaller than
\(q_{\mathrm{requested},j}\).

The fitted map
\(\Delta_j\)
assigns every finite observed value deterministically to exactly one ordered
state defined by the retained cuts. For native-discrete variables,
\(\Delta_j\)
is instead the recorded bijective encoding of the distinguishable native
values.

For every column, QILP-0 records its type, the source of that decision, all
available candidate bounds, requested and effective cardinalities, retained
cuts or native states, bin occupancies, numerical tolerances, uncertainty
metadata, warnings, observable identifier, exact grade, qubit support, and
provenance.

\paragraph{Intuition.}
The discretization does not attempt to discover the ``true bins'' of an
observable. Its purpose is more limited and auditable: to construct a finite
symbolic vocabulary whose granularity does not claim distinctions finer than
the observed profile and the declared observational resolution can support.

This preserves an important separation of responsibilities. The discretizer
determines which states can be defensibly distinguished for each observable;
relationships among the resulting observable states are subsequently exposed
to the logic-induction stage. The supervised target is not used to make those
states artificially convenient for the later reconstruction.

\begin{definition}[Admissible agnostic discretization]
	\label{def:admissible-discretization}

	Let
	\(\mathcal S=\{O_1,\ldots,O_d\}\)
	be the selected original observable vocabulary. An observable-wise
	discretization
	
	\begin{equation}
		\Delta
		=
		(\Delta_j)_{j=1}^{d}
		\label{eq:qilp-admissible-discretizer}
	\end{equation}
	
	is admissible for QILP-0 when it satisfies the following contract.
	
	\begin{enumerate}
		\item \textbf{Target independence.}
		The target values are not used to infer column types, cardinalities, cuts,
		state identifiers, or column order.
		
		\item \textbf{Totality, finiteness, and determinism.}
		Every finite reported value receives exactly one state from a finite domain.
		
		\item \textbf{Preservation of native discrete variables.}
		If a column is declared or empirically audited as an exact native discrete
		variable, its distinguishable observed values are preserved rather than
		requantized.
		
		\item \textbf{Bounded continuous resolution.}
		For a genuinely continuous column, the requested cardinality is determined
		from the available column-specific statistical, distinguishability, and,
		when available, uncertainty-based bounds. These bounds are computed without
		the target.
		
		\item \textbf{Numerical non-degeneracy.}
		Candidate cuts that are indistinguishable under the declared numerical or
		backend resolution are collapsed. The effective state count and all observed
		state occupancies are recorded.

		\item \textbf{Traceability and auditability.}
		Every discrete state is linked to its native value or numerical interval,
		the originating observable, its structural grade and qubit support, and the
		evaluation and uncertainty metadata required to interpret it.
	\end{enumerate}
\end{definition}

This definition is a methodological contract rather than a claim that one
universal optimal discretization exists. It specifies when the numerical
profile of each observable can be mapped to a finite and auditable symbolic
domain at a resolution justified by the available observations. The resulting
collection of discrete observable states provides the vocabulary on which the
joint declarative structure is subsequently induced.

The contract deliberately allows the defensible resolution to depend on the
information available from the observational mechanism. In the absence of
backend- or provider-specific uncertainty information, QILP-0 relies on the
empirical profile and the declared numerical fallback tolerance. When a
backend or observable provider supplies an applicable uncertainty,
resolution, or approximation bound, that information can further restrict the
symbolic resolution. In both cases, the source of the applicable resolution
criterion is retained as provenance rather than hidden inside the
discretization procedure. \\

For row
\(i\),
let

\begin{equation}
	\mathbf x_i^{\mathrm{sym}}
	=
	\left(
	x_{i1},\ldots,x_{id}
	\right)
	\in
	\mathbb R^d
\end{equation}

be its selected continuous observable vector. The fitted column-wise
discretizer produces the finite state

\begin{equation}
	\mathbf z_i
	=
	\Delta
	\left(
	\mathbf x_i^{\mathrm{sym}}
	\right)
	=
	\left(
	\Delta_1(x_{i1}),
	\ldots,
	\Delta_d(x_{id})
	\right).
	\label{eq:qilp-discrete-observational-state}
\end{equation}

At this point the target has still not participated in constructing the
symbolic representation. It enters only in the compatibility audit of the
next condition and, if that audit succeeds, in the subsequent inductive logic
stage.
Accordingly, the fixed numerical floor used by the reference implementation
is not part of the semantics of QILP-0 itself: the contract requires a
declared and traceable resolution criterion, which may be refined through
backend- or provider-specific information as in
Equation~\eqref{eq:qilp-native-tolerance-decomposition}.

\subsection{Condition 9: the discretized representation is twin-admissible for the observed target}
\label{subsec:condition-twin-admissibility}

An exact deterministic declarative reconstruction is possible only if the
target remains well defined after the independently fitted discretization.
QILP-0 therefore performs a compatibility audit only after
\(\Delta\)
has been fixed.

\begin{definition}[Twin-admissibility on the analysed dataset]
	\label{def:twin-admissible}
	
	Let
	\(\mathbf x_i^{\mathrm{sym}}\in\mathbb R^d\)
	be the selected continuous observational vector associated with analysed row
	\(i\), let
	
	\[
	\mathbf z_i
	=
	\Delta(\mathbf x_i^{\mathrm{sym}})
	\]
	
	be its finite discrete observational state, and let
	\(y_i\)
	denote the observed target associated with that row.
	
	An admissible agnostic discretization
	\(\Delta\)
	is twin-admissible for the analysed dataset and selected observable
	vocabulary when
	
	\begin{equation}
		\mathbf z_i
		=
		\mathbf z_j
		\quad\Longrightarrow\quad
		y_i=y_j
		\qquad
		\text{for all }i,j.
		\label{eq:qilp-target-consistency}
	\end{equation}
	
	Equivalently,
	
	\begin{equation}
		\Delta(\mathbf x_i^{\mathrm{sym}})
		=
		\Delta(\mathbf x_j^{\mathrm{sym}})
		\quad\Longrightarrow\quad
		y_i=y_j.
	\end{equation}
\end{definition}

This condition is audited after discretization; it is not used to choose
column types, cardinalities or cuts.

\paragraph{Intuition.}
QILP-0 does not tune the discretization until target conflicts disappear.
It first builds the symbolic vocabulary without the target and then asks
whether that independently obtained representation is sufficiently
discriminating to support an exact deterministic relation on the analysed
rows. If two rows become symbolically indistinguishable but require different
targets, the current configuration is not twin-admissible and QILP-0 must not
claim an exact observational declarative twin for it.

\begin{definition}[Observed discrete relation]
	\label{def:observed-discrete-relation}
	Let
	
	\begin{equation}
		\mathcal D
		=
		\{
		(x_i,y_i)
		\}_{i=1}^{m}
	\end{equation}
	
	be the finite analysed dataset. Once the selected original observable
	vocabulary
	\(\mathcal S\)
	and an admissible discretization
	\(\Delta\)
	have been fixed, the relation presented to the symbolic learner is
	
	\begin{equation}
		\mathcal R_{C,\mathcal D,\mathcal S,\Delta}
		=
		\left\{
		\left(
		\mathbf z_i,y_i
		\right)
		:
		i=1,\ldots,m
		\right\},
		\qquad
		\mathbf z_i
		=
		\Delta(\mathbf x_i^{\mathrm{sym}}).
		\label{eq:qilp-observed-discrete-relation}
	\end{equation}
\end{definition}

If repeated rows induce the same pair
\((\mathbf z_i,y_i)\),
they represent repeated observations of the same relation element. If the
same discrete state occurs with incompatible targets, the finite observed
relation is still constructed, but it is not deterministic with respect to the
target. QILP-0 records the conflicting states and the failed consistency audit
in the execution certificate. Condition~9 is therefore not satisfied and no
exact observational declarative twin is claimed for that configuration. In
the exact-twin branch, this audit terminates the construction before PRIDE
induction, while preserving the non-deterministic relation and its provenance
as an auditable result.

\subsection{Condition 10: the declarative engine can reconstruct the observed finite relation}
\label{sec:condition-PRIDE}

The previous conditions establish when the observational behaviour of the
analysed quantum circuit can be represented as a finite, target-consistent
relation over semantically traceable observable states. The final requirement
of the declarative layer concerns the inductive engine itself: given such a
finite relation, the engine must be able to construct a finite logic program
that reconstructs the observed target relation completely and correctly while
retaining the required notion of rule-body minimality.

QILP-0 instantiates this requirement through PRIDE, within the LFIT family.
Accordingly, the role of PRIDE in QILP-0 is not to discover the observational
vocabulary, determine its resolution, or repair target inconsistencies. Those
operations have already been completed by the preceding target-independent
stages. PRIDE receives the finite relation produced by Conditions~8 and~9 and
induces the declarative program used in the observational-twin construction.

\subsubsection*{PRIDE as inductive propositional logic engine}

GULA \cite{TRILP18,TRMLJ2020} and PRIDE \cite{ISTE2020} are particular implementations of the LFIT framework \cite{InoueRibeiroSakama2014}. 

In this work, PRIDE is used as the LFIT engine for the QILP-0 declarative
layer because it can induce complete and correct multi-valued propositional
rule theories from the finite relations constructed by the preceding stages.

In the present section we introduce notation and describe the fundamentals of both methods.

In the following, we denote by $\N := \{ 0, 1, 2, ... \}$
the set of natural numbers,
and for all $k, n \in \N$, $\segm{k}{n} := \{ i \in \N \mid k \leq i \leq n \}$
is the set of natural numbers between $k$ and $n$ included.
For any set $S$, the cardinal of $S$ is denoted $\card{S}$
and the power set of $S$ is denoted $\powerset(S)$.

Let $\V=\{\var_1,\dots,\var_n\}$ be a finite set of $n \in \N$ variables,
$\Val$ the set in which variables take their values and
$\dom : \V \rightarrow \powerset(\Val)$
a function associating a domain to each variable.
The atoms of \mvl\ (multi-valued logic) are of the form \vv\ where $\var\in\V$ and $\val\in\dom(\var)$.
The set of such atoms is denoted by $\A^{\V}_{\dom} = \{\vv \in \V \times \Val \mid \val \in \dom(\var) \}$
for a given set of variables $\V$
and a given domain function $\dom$.
In the following, we work on specific $\V$ and $\dom$
that we omit to mention when the context makes no ambiguity,
thus simply writing $\A$ for $\A^{\V}_{\dom}$.

\begin{example}
	For a system of 3 variables, the typical set of variables is $\V = \{ a, b, c \}$.
	In general, $\Val = \N$ so that domains are sets of natural integers, for instance:
	$\dom(a) = \{ 0, 1 \}$,
	$\dom(b) = \{ 0, 1, 2 \}$ and
	$\dom(c) = \{ 0, 1, 2, 3 \}$.
	Thus, the set of all atoms is:
	$\A = \{ a^0, a^1, b^0, b^1, b^2, c^0, c^1, c^2, c^3 \}$.
\end{example}

A \mvl\ rule $R$ is defined by:
\begin{equation}\label{discrete_rule}
	R ~~=~~ \vvi{0} \leftarrow \vvi{1} \wedge \cdots \wedge \vvi{m}
\end{equation}
where $\forall i \in \segm{0}{m}, \vvi{i} \in \A$ are atoms in \mvl{}
so that every variable is mentioned at most once in the right-hand part:
$\forall j,k \in \segm{1}{m}, j \neq k \Rightarrow \var_j \neq \var_k$.
Intuitively, the rule $R$ has the following meaning: the variable $\var_0$ can take the value $\val_0$ in the next dynamical step if for each $i \in \segm {1}{m}$, variable $\var_i$ has value $\val_i$ in the current dynamical step.

The atom on the left-hand side of the arrow is called the {\em head} of $R$ and is denoted $h(R) := \vvi{0}$.
The notation $\hvar{R} := \var_0$ denotes the variable that occurs in $h(R)$.
The conjunction on the right-hand side of the arrow is called the {\em body} of $R$, written $b(R)$
and can be assimilated to the set $\{\vvi{1},\dots,\vvi{m}\}$;
we thus use set operations such as $\in$ and $\cap$ on it.
The notation $\bvar{R} := \{ \var_1, \cdots, \var_m \}$ denotes the set of variables that occur in $b(R)$.

More generally, for any set of atoms $X \subseteq \A$, we denote $\setvar{X} := \{ \var \in \V \mid \exists \val \in \dom(\var), \vv \in X \}$ the set of variables appearing in the atoms of $X$.

A {\it multi-valued logic program} (\mvlp) is a set of \mvl\ rules.

\pref{def:domination} introduces a domination relation between rules
that defines a partial anti-symmetric ordering.
Rules with the most general bodies dominate the other rules.
In practice, these are the rules we are interested in since they cover the most general cases.

\begin{definition}[Rule Domination]
	\label{def:domination}
	Let $R_1$, $R_2$ be two \mvl\ rules.
	The rule $R_1$ {\em dominates} $R_2$, written $R_2 \leq R_1$ if $h(R_1) = h(R_2)$ and $b(R_1)\subseteq b(R_2)$.
\end{definition}

In \cite{TRMLJ2020}, the set of variables is divided into two disjoint subsets: $\T$ (for targets) and $\F$ (for features).
This distinction makes it possible to define dynamic multi-valued logic
programs that capture the dynamics of the problem considered in this paper.

\begin{definition}[Dynamic \mvlp]
	Let $\T \subset \V$ and $\F \subset \V$ such that $\F = \V \setminus \T$.
	A \dmvlp\ $P$ is a \mvlp\ such that $\forall R \in P,
	\hvar{R} \in \T$ and $\forall \vv \in b(R), \var \in \F$.
\end{definition}

The dynamical system we want to learn the rules of is represented by a succession of {\em states} as formally given by \pref{def:discrete_state}.
We also define the “compatibility” of a rule with a state in \pref{def:matching}.

\begin{definition}[Discrete state]
	\label{def:discrete_state}
	A {\em discrete state} $s$ on $\T$ (\resp $\F$) of a \dmvlp\ is a function from $\T$ (\resp $\F$) to $\mathbb{N}$, \ie it associates an integer value to each variable in $\T$ (\resp $\F$).
	It can be equivalently represented by the set of atoms
	$\{ \var^{s(\var)} \mid \var \in \T\text{ (\resp \F)} \}$
	and thus we can use classical set operations on it.
	We write \SallT\ (\resp \SallF) to denote the set of all discrete states of $\T$ (\resp $\F$),
	and a couple of states $(s,s') \in \SallF \times \SallT$ is called a \emph{transition}.
\end{definition}

\begin{definition}[Rule-state matching]
	\label{def:matching}
	Let $s \in \SallF$.
	The \mvl\ rule $R$ {\em matches} $s$, written $R\sqcap s$, if $b(R) \subseteq s$.
\end{definition}

In the present setting, an LFIT transition corresponds to one observed
feature--target pair in the finite relation presented to the learner.

When a rule matches a feature state, its body provides a sufficient condition
for the corresponding target conclusion within the learned relation.

The final program we want to learn should both:
\begin{itemize}
	\item match the observations in a complete (all observations are explained) and correct (no spurious explanation) way;
	\item represent only minimal necessary interactions (according to Occam's razor: no overly-complex bodies of rules)
\end{itemize}
GULA \cite{TRILP18,TRMLJ2020} and PRIDE \cite{ISTE2020} can produce such programs.

Formally, given a set of observations $T$, GULA \cite{TRILP18,TRMLJ2020} and PRIDE \cite{ISTE2020} will learn a set of rules $P$ such that all observations are explained: $\forall (s,s') \in T, \forall \vv \in s', \exists R \in P, R \sqcap s, h(R) = \vv$.
All rules of $P$ are correct w.r.t. $T$: $\forall R \in P, \forall (s1,s2) \in T, R \sqcap s1 \implies \exists (s1,s3) \in T, h(R) \in s3$ (if $T$ is deterministic, $s2 = s3$).
All rules are minimal w.r.t. $\F$: $\forall R \in P, \forall R' \in \mvlp, R'$ correct w.r.t. $T$ it holds that $R \leq R' \implies R' = R$.

The possible explanations of an observation are the rules that match the feature state of this observation.
The body of the rules gives minimal condition over feature variables to obtain its conclusions over a target variable.
Multiple rules can match the same feature state, thus multiple explanations can be possible.
Rules can be weighted by the number of observations they match to assert their level of confidence.
Output programs of GULA and PRIDE can also be used in order to predict and explain from unseen feature states by learning additional rules that encode when a target variable value is not possible as shown in the experiments of \cite{TRMLJ2020}.

\subsection{The observational declarative twin as a contractual object}
\label{subsec:observational-declarative-twin-definition}

Having established the conditions under which a finite, target-consistent and
traceable observational relation can be constructed, we can now state
precisely what QILP-0 considers an observational declarative twin.

The object reconstructed by QILP-0 is not the quantum producer \(C\) in
isolation. For fixed
\((C,\mathcal D,\mathcal S,\Delta)\),
Definition~\ref{def:observed-discrete-relation} constructs
\(\mathcal R_{C,\mathcal D,\mathcal S,\Delta}\), which couples the discrete
states obtained from the selected observables with the target values observed
on the analysed dataset. We refer to this finite object as the
\emph{task-conditioned observational relation induced by \(C\)} under the
declared observational scope. The qualifier \emph{task-conditioned} refers to
the presence of the target in the final relation; it does not alter the
target-independent construction of the observable vocabulary, geometry,
selection, or discretization.

\begin{definition}[Observational declarative twin]
	\label{def:observational-declarative-twin}
	A finite multi-valued logic program
	\(\mathcal T\)
	is an observational declarative twin of the task-conditioned observational
	relation induced by the circuit or quantum producer
	\(C\)
	under the declared scope
	\((\mathcal D,\mathcal O_{\leq K_{\mathrm{ref}}},\mathcal S,\Delta)\)
	when:
	
	\begin{enumerate}
		\item \(\mathcal T\) is complete and correct for
		\(\mathcal R_{C,\mathcal D,\mathcal S,\Delta}\);
		
		\item the induced program is a sufficient set of rules whose bodies are
		minimal with respect to the observed relation;
		
		\item every rule literal is traceable to an original selected observable,
		its discrete state or interval/native value, its structural grade, and the
		qubits on which the observable acts; and
		
		\item the declared equivalence scope explicitly records the analysed dataset,
		observable provider and horizon, stopping policy, observable-selection
		contract, discretization policy, backend or estimator, and relevant numerical
		or uncertainty metadata.
	\end{enumerate}
\end{definition}

When the declared scope is clear, we use the shorter expression
\emph{observational declarative twin of \(C\)} for this relation-qualified
object. The abbreviation must not be interpreted as equivalence to the complete
unitary semantics of \(C\), nor as a claim that the induced theory reproduces
every possible input--output behaviour of a predictor attached to \(C\). The
equivalence is observational and finite-scope: it concerns
\(\mathcal R_{C,\mathcal D,\mathcal S,\Delta}\) under the declared scope. It
also does not imply that the induced program is globally minimal in number of
rules.

\paragraph{Interpretation.}
The definition acts as a contract between the continuous and declarative
parts of QILP-0. Conditions 1--9 determine what has legitimately reached the
finite relation; the logic-induction stage must then reconstruct that relation
while preserving the provenance required to interpret its literals. The next
section shows that the implemented QILP-0 pipeline realizes this contract
constructively.

\section{Constructive realization and correctness of QILP-0}
\label{sec:constructive-realization}

The previous sections separated the QILP-0 construction into explicit
methodological conditions and defined the finite-scope object that the method
is intended to produce. We now make the construction operational.

The algorithms below are not introduced as an alternative formulation of the
method. They are the constructive realization of the contracts already
defined: target-independent observational discovery, preservation of original
observable semantics, column-wise agnostic discretization, post-discretization
twin-admissibility auditing, logic induction, and certification.

They specify the scientific behaviour required of an implementation rather
than one particular software API. Equivalent numerically stable
implementations may be used when they preserve the same observational scope,
subspaces, stopping semantics, original-variable traceability, discretization
contract, and recorded provenance.

\begin{algorithm}[t]
	\caption{QILP-0 observational-declarative-twin pipeline}
	\label{alg:qilp-main}
	\begin{algorithmic}[1]
		\Require Analysed dataset \(\mathcal D\); circuit or quantum producer \(C\);
						declared observable provider; reference horizon
		\(K_{\mathrm{ref}}\);
		exact-grade families
		\(\{\mathcal O^{(k)}\}_{k=1}^{K_{\mathrm{ref}}}\);
		stopping policy \(\pi\);
		discretization configuration \(\Theta_\Delta\)
		\Ensure Theory \(\mathcal T\), metadata \(\mathcal M\), and certificate
		\(\Gamma\)

		\State Obtain the complete observational blocks required by
		\(\mathcal O_{\leq K_{\mathrm{ref}}}\)
		and build the fixed reference
		\((X_{\mathrm{ref}},G_{\mathrm{ref}},M_{\mathrm{ref}})\)
		
		\State
		\((\mathcal S,X_{\mathrm{sym}},K_{\mathrm{end}},
		\mathcal M_{\mathrm{geom}})
		\gets
		\Call{IncrObsDiscovery}
		{\mathcal D,C,\{\mathcal O^{(k)}\},K_{\mathrm{ref}},
			X_{\mathrm{ref}},G_{\mathrm{ref}},\pi}\)

		\State
		\((Z,\Delta,\mathcal M_\Delta)
		\gets\)
		\Statex
		\(
		\Call{AdmissibleColumnwiseDiscretization}
		{X_{\mathrm{sym}},\mathcal S,\Theta_\Delta}\)
		
		\State
		\(\mathcal R
		\gets
		\{(Z_i,y_i):(x_i,y_i)\in\mathcal D\}\)
		
		\State
		\((\mathcal T,\mathcal M_{\mathrm{sym}},\Gamma)
		\gets
		\Call{InduceAndCertifyTwin}
		{\mathcal R,\mathcal S,\mathcal M_\Delta}\)
		
		\State
		\(\mathcal M
		\gets
		\mathcal M_{\mathrm{geom}}
		\cup
		\mathcal M_\Delta
		\cup
		\mathcal M_{\mathrm{sym}}\)
		
		\State \Return
		\((\mathcal T,\mathcal M,\Gamma)\)
	\end{algorithmic}
\end{algorithm}

\begin{algorithm}[t]
	\caption{Target-independent incremental observable discovery}
	\label{alg:qilp-observable-discovery}
	\small
	\begin{algorithmic}[1]
		\Require \(\mathcal D\), \(C\), exact-grade families
		\(\{\mathcal O^{(k)}\}_{k=1}^{K_{\mathrm{ref}}}\),
		fixed reference
		\((X_{\mathrm{ref}},G_{\mathrm{ref}},M_{\mathrm{ref}})\),
		and stopping/resource policy \(\pi\)
		\Ensure Selected original observables \(\mathcal S\), their continuous matrix
		\(X_{\mathrm{sym}}\), last completed support \(K_{\mathrm{end}}\), and
		geometric metadata \(\mathcal M_{\mathrm{geom}}\)
		
		\State \(\mathcal S\gets\emptyset\);
		\(Q\gets\) empty accumulated orthonormal basis;
		\(K_{\mathrm{end}}\gets0\)
		
		\For{\(k=1,\ldots,K_{\mathrm{ref}}\)}
		
		\State Access the complete response block associated with
		\(\mathcal O^{(k)}\) under the same provider semantics used for reference
		calibration
		
		\State Reuse persisted responses when available, or retrieve/recompute an
		equivalent block without changing the fixed reference
		
		\State Construct the valid normalized geometric block according to
		Condition~4
		
		\State Residualize the current block against the previously accumulated
		subspace \(Q\)
		
		\State Compute the residual thin SVD or an equivalent numerically stable
		subspace update

		\State Retain the smallest leading set of numerically resolvable new
		directions satisfying the declared \(\tau_{\mathrm{dir}}\) criterion and
		form the candidate basis
		\(Q_{\mathrm{cand}}=[Q\;Q_{\mathrm{add}}^{(k)}]\)

		\State Compute the association score
		\(s_j^{(k)}\)
		of every valid original grade-\(k\) column with the candidate new subspace
		according to Condition~7
		
		\State Rank candidate columns by decreasing association score, breaking exact
		ties by the stable target-independent structural ordinal
		
		\State Select the smallest ranked prefix whose cumulative score reaches
		\(\tau_{\mathrm{col}}\), obtaining
		\(\mathcal J_k\)
		
		\State Evaluate the configured downstream admissibility of the candidate
		accumulation
		\(\mathcal S\cup\mathcal J_k\)
		
		\If{the candidate accumulation violates the downstream resource policy}
		\State Persist the candidate geometric/selection diagnostics and the
		resource endpoint reason
		\State \textbf{break}
		\EndIf
		
		\State Commit \(Q\gets Q_{\mathrm{cand}}\)
		
		\State Update the represented mass and
		\(C_{\mathrm{ref}}(k)\)
		against the fixed reference
		
		\State Append the original columns indexed by
		\(\mathcal J_k\)
		to
		\(\mathcal S\)
		using the declared target-independent structural order
		
		\State Set \(K_{\mathrm{end}}\gets k\)
		
		\State Persist support-\(k\) geometric quantities, direction- and
		column-retention thresholds, selection scores, selected source-column
		identifiers, structural ordinals, and provider provenance
		
		\If{\(\pi\) stops after the completed grade \(k\) by coverage, plateau,
			or reference-horizon exhaustion}
		\State Persist the endpoint reason
		\State \textbf{break}
		\EndIf
		
		\EndFor
		
		\State Load the original continuous values of the accumulated observables in
		\(\mathcal S\) as \(X_{\mathrm{sym}}\)
		
		\State \Return
		\((\mathcal S,X_{\mathrm{sym}},K_{\mathrm{end}},
		\mathcal M_{\mathrm{geom}})\)
	\end{algorithmic}
\end{algorithm}

\begin{algorithm}[t]
	\caption{Admissible target-independent column-wise discretization}
	\label{alg:qilp-discretization}
	\begin{algorithmic}[1]
		\Require Continuous matrix \(X_{\mathrm{sym}}\), selected-observable metadata
		\(\mathcal S\), configuration \(\Theta_\Delta\)
		\Ensure Discrete matrix \(Z\), fitted discretizer \(\Delta\), diagnostics
		\(\mathcal M_\Delta\)
		
		\For{each selected observable column \(O_j\)}
		\State Infer the column kind from declared metadata and the empirical audit
		\If{\(O_j\) is constant}
		\State Assign its unique state and record the degeneracy
		\ElsIf{\(O_j\) is exact native discrete}
		\State Preserve its distinguishable observed values as states
		\Else
		\State Estimate \(q_{\mathrm{FD},j}\) and
		\(q_{\mathrm{distinct},j}\)
		\State Estimate \(q_{\mathrm{SNR},j}\) when uncertainty metadata exist
		\State Compute \(q_{\mathrm{requested},j}\) using
		Equation~\eqref{eq:qilp-requested-cardinality}
		\State Compute target-independent quantile cuts for column \(j\)
		\State Collapse cuts indistinguishable under the declared
		numerical/backend resolution
		\State Assign states deterministically using the ordered effective cuts
		\EndIf
		\State Persist effective cardinality, cuts or native values, state
		occupancies, warnings, tolerances, and observable provenance
		\EndFor
		
		\State Concatenate the discretized columns in the stable target-independent
		structural order
		
		\State \Return
		\((Z,\Delta,\mathcal M_\Delta)\)
	\end{algorithmic}
\end{algorithm}

\begin{algorithm}[t]
	\caption{PRIDE induction and observational-twin certification}
	\label{alg:qilp-certification}
	\begin{algorithmic}[1]
		\Require Observed relation \(\mathcal R\), selected observables \(\mathcal S\),
		discretization metadata \(\mathcal M_\Delta\)
		\Ensure Theory \(\mathcal T\), symbolic metadata, and certificate \(\Gamma\)
		
		\State Group identical discrete feature states and audit their target values
		
		\If{one discrete state has incompatible targets}
		\State Record the conflicting states and reject the exact-twin claim for
		the current configuration
		\State \Return audited non-twin result
		\EndIf
		
		\State Fit PRIDE to \(\mathcal R\) using the same stable 		feature order

		\State Parse every induced rule and map every literal to its
		\((\text{observable},\text{grade},\text{qubits},\text{state},
		\text{interval/native value})\) provenance
		
		\For{each observed row \((Z_i,y_i)\)}
		\State Evaluate the induced rules on \(Z_i\)
		\State Record coverage, unique correctness, uncovered cases, and conflicts
		\EndFor
		
		\State Compute strict reconstruction accuracy and theory hash
		\State Verify completeness and consistency over all observed rows
		
		\State Build
		\(\Gamma
		=
		(\Gamma_\Delta,\Gamma_{\mathrm{LFIT}},\Gamma_{\mathrm{backend}})\)
		
		\State \Return
		\((\mathcal T,\mathcal M_{\mathrm{sym}},\Gamma)\)
	\end{algorithmic}
\end{algorithm}

\paragraph{Certificate contents.}
The QILP-0 certificate records the evidence required to audit the scope and
correctness claim of a run. In the current construction this includes:
target-independence metadata; native-versus-continuous column treatment;
requested and effective discretization cardinalities; retained or collapsed
cuts; observed state occupancies; numerical tolerances and
sample-sufficiency warnings; provider/backend provenance; target-consistency
of the discrete relation; uncovered or conflicting rows; strict
reconstruction accuracy; theory completeness and consistency; a theory hash;
and literal-to-observable traceability.

Sample-sufficiency warnings are diagnostic flags under the current policy,
not failed certificate conditions. Likewise, backend uncertainty qualifies
the relationship between reported observable values and ideal expectations;
it does not change the exactness of the logical reconstruction with respect
to the reported discretized relation.

\subsection{Constructive correctness}
\label{subsec:qilp-constructive-correctness}

The role of the following result is compositional. It does not introduce an
independent characterization of quantum-circuit semantics; rather, it checks
that the methodological conditions and constructive steps defined above
compose into the contractual object of
Definition~\ref{def:observational-declarative-twin}.

\begin{proposition}[Constructive correctness of QILP-0]
	\label{prop:qilp-constructive-correctness}
	
	Let
	\(\mathcal D\)
	be finite and let the declared observable horizon be finite. Assume that the
	observable evaluations required by the completed traversal are defined under
	the declared provider/backend contract and that Conditions 1--7 return a
	finite set
	\(\mathcal S\)
	of original traceable observables.
	
	If Condition 8 yields an admissible agnostic discretization
	\(\Delta\)
	and Condition 9 verifies that the resulting observed relation is
	twin-admissible, then the QILP-0 construction of
	Algorithms~\ref{alg:qilp-main}--\ref{alg:qilp-certification}
	terminates and returns a finite multi-valued logic program
	\(\mathcal T\)
	that satisfies
	Definition~\ref{def:observational-declarative-twin}
	within the declared observational scope.
\end{proposition}

\begin{proof}
	The analysed dataset and the declared observational reference horizon are
	finite. Conditions 1--7 organize that finite reference horizon into complete
	exact-grade blocks and define a traversal with an explicit endpoint.
 Therefore the continuous
	construction terminates after finitely many completed grades and returns a
	finite selected vocabulary
	\(\mathcal S\)
	of original observables together with their provenance.
	
	By Condition 8, the fitted discretizer
	\(\Delta\)
	acts independently on each selected column, is deterministic and
	target-independent, and maps every finite observed response to one state in a
	finite domain. Hence the set of discretized observational states is finite.
	Condition 9 is evaluated only after
	\(\Delta\)
	has been fixed. When it succeeds, identical discrete feature states have
	identical targets, so
	\(\mathcal R_{C,\mathcal D,\mathcal S,\Delta}\)
	is a finite deterministic target relation.
	
	LFIT provides logic-program constructions for finite deterministic
	interpretation-transition relations, with the corresponding soundness and
	completeness properties established for that framework
	\cite{InoueRibeiroSakama2014}. In the QILP-0 explanatory encoding, each
	observed feature--target pair is represented as such a finite relation.
	PRIDE returns in polynomial time a sufficient set of rules whose bodies are
	minimal with respect to the supplied observations and which completely and
	correctly explain that relation
	\cite{OrtegaEtAl2021Symbolic}.
	
	Applying PRIDE to
	\(\mathcal R_{C,\mathcal D,\mathcal S,\Delta}\)
	therefore yields a finite program
	\(\mathcal T\)
	that is complete and correct for the observed discrete relation, with
	rule-body minimality in the stated PRIDE sense. QILP-0 additionally preserves
	the mapping from every rule literal to its discrete state, interval or native
	value, original observable, structural grade, qubit support, and evaluation
	metadata.
	
	Thus each clause of
	Definition~\ref{def:observational-declarative-twin}
	is satisfied, and
	\(\mathcal T\)
	is an observational declarative twin within the declared finite scope.
\end{proof}

\begin{corollary}[Exact-backend observational scope]
	\label{cor:qilp-exact-backend}
	
	If the observable evaluations used to construct
	\(X_{\mathrm{sym}}\)
	are exact, then the resulting symbolic equivalence is exact with respect to
	the discretized relation obtained from those exact observable responses within
	the declared dataset, provider, horizon, selection, and discretization scope.
\end{corollary}

\begin{remark}[Estimated-backend observational scope]
	\label{rem:qilp-estimated-backend}
	
	If observable responses are numerical or sampling-based estimates, QILP-0
	still requires exact logical reconstruction of the reported discretized
	relation whenever an exact observational declarative twin is claimed. The
	relationship between those reported values and ideal quantum expectations is
	instead qualified by the provider/backend uncertainty, numerical tolerances,
	and discretization metadata recorded in the certificate.

\end{remark}

\begin{remark}[Failure of twin-admissibility]
	\label{rem:qilp-twin-admissibility-failure}
	
	If Equation~\eqref{eq:qilp-target-consistency} fails, the independently
	constructed discretized representation does not define a deterministic target
	relation on the analysed dataset. QILP-0 records the conflicting states and
	does not claim an exact observational declarative twin for that configuration.
	
	A different provider, observational reference horizon, independently
	specified discretization policy, or other declared configuration may be
	evaluated in a separate run,
	but the target conflicts of the current run are not used
	retroactively to modify its cuts.
\end{remark}

\subsection{Computational scope and scalability}
\label{subsec:qilp-computational-scope}

The computational cost of QILP-0 should not be treated as a single
undifferentiated quantity. The pipeline separates at least four computational
stages: observable acquisition, geometric processing, discretization, and
declarative induction. These stages depend on different parameters and may be
executed under substantially different computational regimes. In particular,
an efficient downstream representation does not retroactively remove the cost
of observables that had to be acquired in order to construct it.

For the Pauli provider used in the present experiments, the number of
non-identity Pauli words at exact support \(k\) on \(n\) qubits is
\begin{equation}
	N_k(n)
	=
	\binom{n}{k}3^k,
	\label{eq:qilp-pauli-exact-support-count}
\end{equation}
and the number contained in a declared horizon \(K\leq n\) is
\begin{equation}
	N_{\leq K}(n)
	=
	\sum_{k=1}^{K}\binom{n}{k}3^k.
	\label{eq:qilp-pauli-horizon-count}
\end{equation}
For fixed \(K\), this quantity is \(O(n^K)\); an unrestricted traversal through
\(K=n\) contains \(4^n-1\) non-identity Pauli words. Consequently, a
support-bounded observational reference horizon is computationally different
from an unrestricted traversal of the full Pauli observable family.

For the Pauli provider,
\(K_{\mathrm{phys}}=n\)
is therefore a structural maximum, not a promise that every execution can or
should evaluate the complete family through support \(n\).
QILP-0 separates this provider-defined maximum from the declared reference
horizon \(K_{\mathrm{ref}}\). The latter is part of the observational scope
supplied to the pipeline and may equal \(K_{\mathrm{phys}}\) whenever the full
provider horizon is available and scientifically intended.

The number of declared observables does not, by itself, determine acquisition
cost. Each observable provider and execution environment also determine the
cost of evaluating a complete structural grade. Depending on the application,
observable responses may already be available in a persisted dataset, may be
obtained from quantum hardware, or may be computed through analytical,
statevector, tensor-network, multicore, GPU-accelerated, or other simulation
backends. QILP-0 therefore does not assume a unique computational architecture
for observable acquisition.

QILP-0 does not prescribe how
\(K_{\mathrm{ref}}\)
must be chosen. An application may provide it directly, for example because a
domain analysis defines the support range of interest or because an existing
observational dataset is available only through a particular grade.
A provider
or execution environment may likewise impose an effective technological
horizon.

When an installation chooses to determine
\(K_{\mathrm{ref}}\)
automatically from computational constraints, it may use an independent
prospective resource-assessment policy
\begin{equation}
	A_{\mathrm{ref}}
	\left(
	K;\mathcal P,\mathcal D,\mathcal E,\mathcal B_{\mathrm{ref}}
	\right)
	\in\{0,1\},
	\label{eq:qilp-reference-resource-admissibility}
\end{equation}
where the policy estimates whether all complete grades through \(K\) can be
used to construct the declared reference. This assessment precedes observable
acquisition for the reference and need not compute the corresponding
observable responses in order to estimate their structural cost.

Once
\(K_{\mathrm{ref}}\)
has been declared and the fixed reference has been calibrated, a separate
processing-admissibility policy may govern whether the next complete grade can
be propagated through the incremental and downstream stages:
\begin{equation}
	A_{\mathrm{proc}}
	\left(
	k;\mathcal P,\mathcal E,\mathcal B_{\mathrm{proc}}
	\right)
	\in\{0,1\}.
	\label{eq:qilp-processing-resource-admissibility}
\end{equation}

The two policies have different semantics. A prospective rejection by
\(A_{\mathrm{ref}}\)
prevents a candidate grade from entering the declared reference horizon. A
later rejection by
\(A_{\mathrm{proc}}\)
may stop the incremental traversal at
\(K_{\mathrm{end}}<K_{\mathrm{ref}}\)
even though the rejected grade already contributed to reference calibration.
Neither policy may silently replace a complete structural grade by a partial
subset. The declared reference horizon, last completed support, relevant policy
configuration, and termination reason are recorded as execution metadata.

The geometric stage operates on the observable-response profiles after
acquisition. 
Its cost depends on the number of analysed rows, the number of
valid columns in each completed grade, the retained geometric dimension, and the
chosen numerical realization.
The scientific contract does not require a
single implementation strategy: explicit-Gram, block, streaming, matrix-free,
or equivalent structured SVD realizations may be used provided that they
preserve the reference-relative coverage semantics of the fixed observational
reference, numerical equivalence, orthogonality, and provenance.
Importantly, geometric selection can reduce the
number of original observables propagated to discretization and declarative
induction, but it does not eliminate the acquisition cost of the declared
blocks already evaluated.

Discretization is performed observable-wise on the selected columns and is
therefore naturally separable across variables. Its practical cost depends on
the number of retained observable profiles, the number of analysed rows, and
the resolution policy. The final PRIDE stage depends instead on the finite
relation presented to the inductive engine, including the number of symbolic
variables, their effective cardinalities, and the structure of the observed
relation. Resource requirements should therefore be interpreted stage by
stage rather than attributed to QILP-0 through a single hardware-independent
runtime figure.

This separation also determines the scope of the computational evaluation in
the present methodological study. Wall-clock time and peak-memory measurements
obtained on a single installation would combine properties of the observable
provider, simulator or quantum backend, numerical implementation, resource
policy, and physical computing platform. We therefore report the structural
growth of the declared observational family, the supports actually completed,
and whether termination was caused by the coverage criterion or by the
declared resource policy, rather than presenting installation-specific timings
as general evidence of QILP-0 scalability. A controlled performance study
across providers, execution environments, and computational architectures is a
separate systems question and remains an important direction for future work.

It is important to conclude that the 100-qubit Bars \& Stripes executions demonstrate that the complete
declared support-\(\leq2\) reference can be constructed for these
embedding-specific providers under the reported policy. They do not establish
that the unrestricted Pauli horizon
\(K_{\mathrm{phys}}=100\)
can be evaluated at comparable cost.

\paragraph{AI-assisted software development.}
OpenAI ChatGPT was used during software implementation to accelerate prototype
coding, debugging, and technical documentation. All such assistance was
supervised and reviewed by the authors and was not used to alter experimental
data, results, or scientific conclusions.

\section{Experimental Validation}
\label{sec:results}

This section evaluates the constructive QILP-0 pipeline on two experimental
settings selected to exercise qualitatively different observational regimes.
The objective is not to benchmark quantum predictive performance, but to
verify that the methodological construction developed in the previous
sections can be carried out in practice and that the resulting declarative
theories satisfy the observational-twin contract within their declared scope.

The first case, Bars \& Stripes, provides an exhaustive and controlled
native-discrete setting in which two fixed quantum embeddings can be compared
over increasing system sizes. The second case, Low-Depth MNIST, provides a
genuinely continuous observational setting and allows us to compare the same
data representation before and after a trained variational transformation.
Together, the two cases exercise the native-discrete and continuous branches
of the QILP-0 declarative interface and allow the geometric, discretization,
symbolic, and certification stages to be audited under different conditions.

\subsection{Validation objectives and certificate criteria}
\label{sec:experimental-certificate-criteria}

For each analysed relation, validation separates the continuous observational
stage from the declarative reconstruction stage. 
We report the executed
structural horizon, geometric coverage and retained geometric dimension; the
discretization audit applicable to the selected original observables;
consistency of the resulting discrete relation; and the row-level
reconstruction obtained from the induced PRIDE theory.

These quantities are interpreted only within the declared
observational scope.

Unless otherwise stated, the reported geometric runs used a residual-direction
retention threshold
\(\tau_{\mathrm{dir}}=0.999\)
and an original-column association threshold
\(\tau_{\mathrm{col}}=0.999\).
Constant-column detection used tolerance
\(10^{-12}\);
the residual-zero diagnostic tolerance was
\(10^{-10}\);
and numerical rank or reorthogonalization decisions used tolerance
\(10^{-8}\).
The explicit-residual and column-Gram implementations realize the same
target-independent geometric-selection contract through different numerical
paths; the active configuration and selected-source metadata are recorded for
audit.

An exact observational-twin certificate requires the discretized relation to
be target-consistent and the induced declarative theory to reconstruct that
relation completely and correctly: every analysed row must be covered, no
conflicting prediction may be produced, and strict reconstruction accuracy
must equal one. The certificate additionally preserves the provenance required
to trace symbolic literals back to their discrete states and original
observables.

Backend or provider uncertainty, when present, qualifies the correspondence
between the reported observable values and their ideal values; it does not by
itself turn an exact symbolic reconstruction of the reported relation into an
approximate logical reconstruction. The symbolic equivalence claim therefore
remains exact with respect to the reported discrete relation, while the
correspondence between that relation and the ideal quantum responses remains
qualified by the declared uncertainty metadata and is made explicit for the
domain expert.

\subsection{Native-discrete validation: Bars \& Stripes}
\label{sec:bars-stripes}

\subsubsection{Community benchmark and controlled architectural comparison}

Bars \& Stripes is a synthetic image family widely used as a
proof-of-principle benchmark in quantum generative modelling and quantum
machine-learning studies
\cite{benedetti2019generative,liu2018differentiable,bowles2024subtle}.
We selected it from this community experimental tradition rather than
designing a dataset specifically for QILP-0. An accessible implementation and
derived train/test datasets are also distributed through PennyLane Datasets
\cite{bowles2024pennylaneBas}.

Our experiments use an exhaustive, noise-free binary-classification variant.
For an $L\times L$ image, $2^L-2$ unambiguous bar patterns and $2^L-2$
unambiguous stripe patterns remain after excluding the all-zero and all-one
images, which satisfy both definitions. The evaluated dataset therefore
contains
\begin{equation}
	(2^L-2)+(2^L-2)=2^{L+1}-4.
	\label{eq:bas-cardinality}
\end{equation}
We considered $L=4,\ldots,10$, corresponding to 28--2,044 patterns and
16--100 qubits.

We compare two fixed, non-trainable quantum feature maps for exactly the same
inputs. Both use the same local angle encoding, while the second adds a
nearest-neighbour entangling layer. Rotation-based angle encodings are
standard constructions in the QML literature
\cite{schuldkilloran2019feature,schuld2021encoding}.

The first circuit is the separable product embedding
\begin{equation}
	|\phi_{\mathrm{prod}}(x)\rangle
	=
	\bigotimes_{j=0}^{L^2-1}
	R_Y(\alpha x_j)|0\rangle,
	\qquad
	\alpha=\pi/3.
	\label{eq:bas-product-embedding}
\end{equation}

\begin{figure}[!htbp]
	\centering
	\includegraphics[width=1\linewidth]{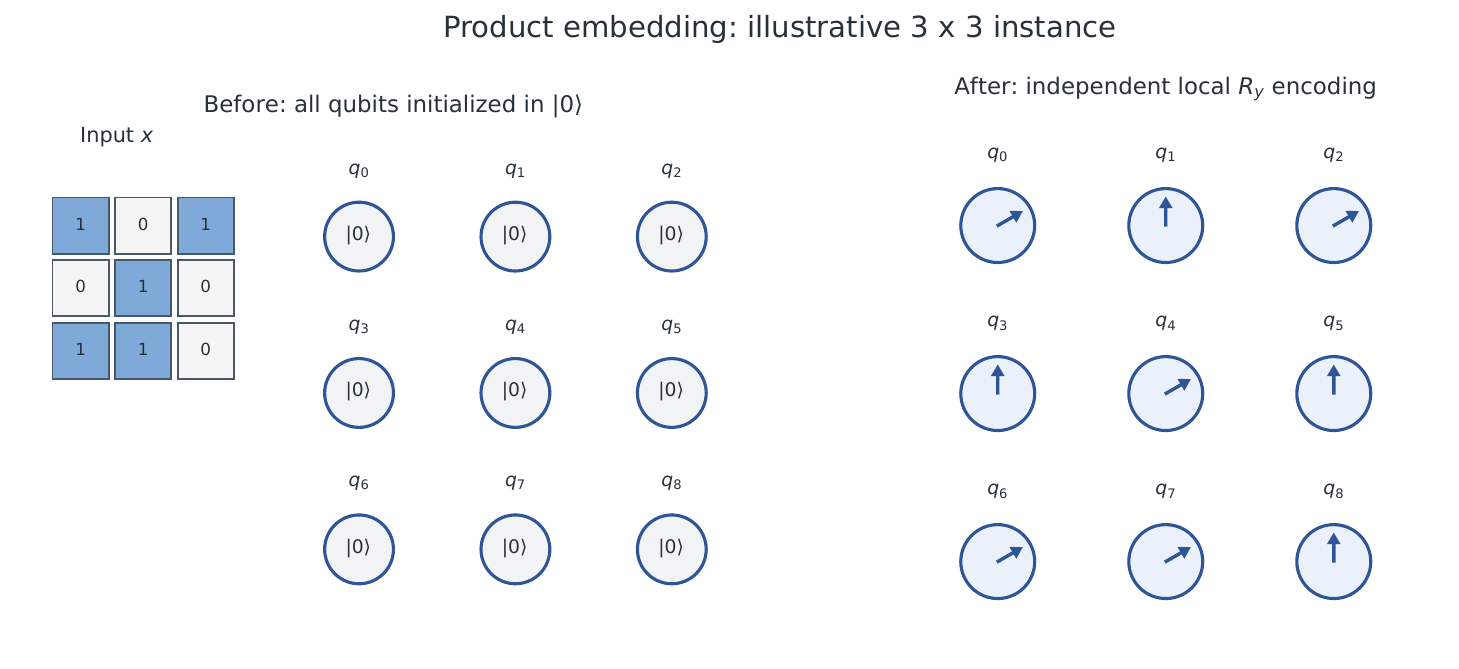}
	\caption{Product embedding as a local state transformation. A
		$3\times3$ instance is shown for readability. Each binary pixel $x_j$ is
		assigned to qubit $q_j$, initially prepared in $|0\rangle$, and independently
		encoded by $R_Y(\alpha x_j)$ with $\alpha=\pi/3$. No two-qubit interaction
		is present. The reported experiments use $L=4,\ldots,10$, rather than the
		illustrative $L=3$ instance.}
	\label{fig:bs-product-before-after}
\end{figure}

The second circuit uses the same local encoding and then applies CZ gates
along all horizontal and vertical nearest-neighbour edges of the open
$L\times L$ grid:
\begin{equation}
	|\phi_{\mathrm{CZ}}(x)\rangle
	=
	U_{\mathrm{CZ}}(G_L)
	\left[
	\bigotimes_{j=0}^{L^2-1}
	R_Y(\alpha x_j)
	\right]
	|0\cdots0\rangle.
	\label{eq:bas-grid-cz-embedding}
\end{equation}

\begin{figure}[!htbp]
	\centering
	\includegraphics[width=1\linewidth]{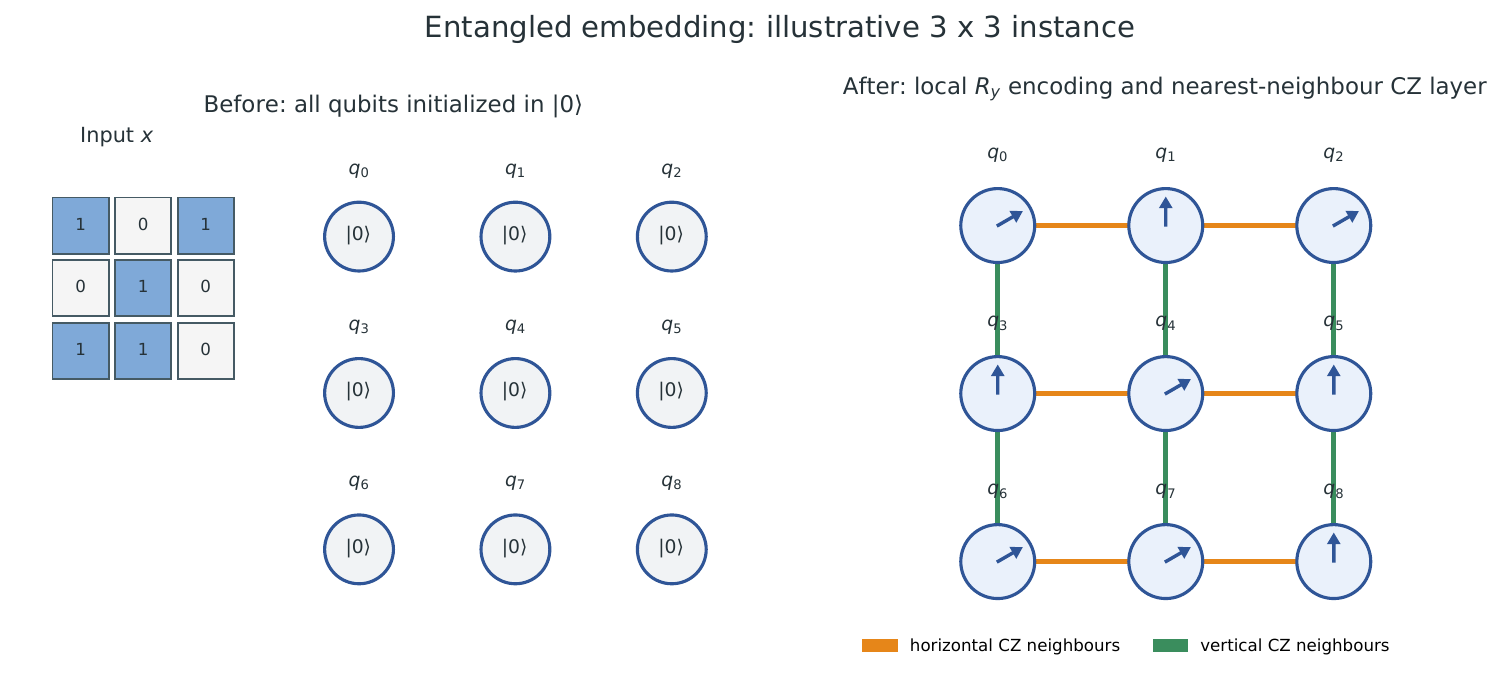}
	\caption{Local encoding followed by grid-neighbour entanglement. The initial
		state and local $R_Y(\alpha x_j)$ rotations are identical to those of the
		product embedding. The encoded qubits are then coupled through CZ gates on
		all horizontal and vertical nearest-neighbour edges. The $3\times3$ grid is
		illustrative; the reported experiments use the corresponding $L\times L$
		connectivity for $L=4,\ldots,10$.}
	\label{fig:bs-entangled-before-after}
\end{figure}

Thus, the controlled comparison changes only the inclusion of elementary
two-qubit interactions after the common local encoding. QILP-0 is used to
examine whether this architectural change modifies the observable-response
geometry and whether any such modification propagates to the induced
declarative relation. The illustrative $3\times3$ diagrams are used only for
legibility; all quantitative results below correspond to the executed
$L=4,\ldots,10$ study.

\subsubsection{Exact observable evaluation and executed supports}

For the product embedding, every local state lies in the $XZ$ plane of the
Bloch sphere, so $\langle Y_j\rangle=0$. Because the complete state is a tensor
product, every Pauli word containing at least one $Y$ has an exactly vanishing
expectation value. Such words therefore disappear from the effective
observational matrix as a direct consequence of the embedding semantics,
without approximation or heuristic filtering. The remaining observables were
evaluated analytically, and the implementation was checked against exact
statevector calculations on validation instances.

For the grid-CZ embedding, the CZ layer is a Clifford unitary. Conjugation by
this unitary therefore maps each Pauli word exactly to another Pauli word,
possibly with a sign \cite{aaronson2004stabilizer}. We evaluated the requested
observables through this Clifford pullback on the underlying product state.
The algebraic transformation introduces neither approximation nor sampling
uncertainty; small-system exhaustive enumeration and exact statevector
calculations were used only to check the implementation.

In the Bars \& Stripes experiments, the reference horizon was not supplied by
an external domain constraint. We therefore instantiated it through a
target-independent prospective resource assessment performed before reference
calibration. For each candidate exact-support grade, the implementation used
the provider-specific number of observables to be evaluated,
\(N_k^{\mathrm{eval}}\), together with the number of dataset rows \(m\), to
estimate the resulting number of scalar observable values
\begin{equation}
	V_k=mN_k^{\mathrm{eval}},
	\label{eq:bs-prospective-values}
\end{equation}
and a conservative Gram-work proxy
\begin{equation}
	W_k=m^2N_k^{\mathrm{eval}}.
	\label{eq:bs-prospective-gram-work}
\end{equation}
The validated implementation also monitored an estimated peak working-memory
requirement with a safety factor of 1.35.

The declared limits were 750,000 observables in a complete grade,
450,000,000 scalar observable values, an
\(8\times10^{11}\) Gram-work proxy, and approximately 8 GB of estimated
working memory. The assessment inspected complete candidate grades
prospectively and chose the largest consecutive support satisfying the
declared policy. It did not use the target or the actual observable-response
values to decide the horizon.

For both Bars \& Stripes routes, the resulting reference horizons for
\(L=4,\ldots,10\) were
\(6,4,3,3,3,2,2\), respectively. The implementation also imposed a prospective
inspection cap of support 8, but this cap did not determine any of the reported
reference horizons because a resource criterion became active first in every
case.

After \(K_{\mathrm{ref}}\) had been fixed, the reference-calibration pass
obtained all complete blocks within that horizon and constructed the fixed
reference. The subsequent incremental traversal could stop at a smaller
\(K_{\mathrm{end}}\) after the reference had been fixed, either because the
0.99 reference-coverage criterion had already been reached or because the
configured downstream selected-observable/declarative-processing budget
rejected the next complete grade.

This distinction is particularly visible for the product runs at
\(L=7\) and \(L=8\). In both cases,
\(K_{\mathrm{ref}}=3\),
so support 3 contributed to the fixed reference during calibration, whereas
the incremental traversal stopped at
\(K_{\mathrm{end}}=2\).
Consequently, the reported coverages 0.978420 and 0.977275 are
\(C_{\mathrm{ref}}(2)\) values measured against the fixed
support-\(\leq3\) reference. They are not obtained by redefining the
denominator at the last processed support.

At 100 qubits, both routes used
\(K_{\mathrm{ref}}=2\)
and completed both grades of that declared reference horizon. The resulting
reference-relative coverages were 1.000000 for the product embedding and
0.999874 for grid-CZ. These values characterize the support-\(\leq2\)
reference used in the reported executions; they do not quantify the amount of
observable geometry that may exist at higher Pauli supports.

Table~\ref{tab:bs-observational-results} reports the actual final support,
coverage, retained geometric dimension, number of original observables
retained for symbolic processing, and stopping condition for every evaluated
size.

\begin{table*}[t]
	\centering
	\scriptsize
	\setlength{\tabcolsep}{3pt}
	\caption{Bars \& Stripes: declared reference horizon, final processed
		support, and observational summary. Coverage is measured relative to the
		fixed reference horizon. ``Coverage'' denotes termination after reaching
		the 0.99 criterion; ``resources'' denotes a downstream post-selection
		endpoint in which the next support was analyzed geometrically but its
		selected original observables were not admitted because they would have
		exceeded the configured cumulative feature budget.}
	\label{tab:bs-observational-results}
	\begin{tabular}{llrrrrrrrl}
		\toprule
		Embedding & $L$ & Patterns & Qubits & $K_{\mathrm{ref}}$ &
		$K_{\mathrm{end}}$ & Ref. coverage & Geom. dim. & Selected obs. & Endpoint \\
		\midrule
		Product $R_Y$ & 4  & 28    & 16  & 6 & 3 & 0.994704 & 26  & 3,784  & coverage \\
		Product $R_Y$ & 5  & 60    & 25  & 4 & 3 & 0.997892 & 49  & 16,193 & coverage \\
		Product $R_Y$ & 6  & 124   & 36  & 3 & 3 & 1.000000 & 81  & 51,644 & coverage \\
		Product $R_Y$ & 7  & 252   & 49  & 3 & 2 & 0.978420 & 55  & 4,780  & resources \\
		Product $R_Y$ & 8  & 508   & 64  & 3 & 2 & 0.977275 & 71  & 8,153  & resources \\
		Product $R_Y$ & 9  & 1,020 & 81  & 2 & 2 & 1.000000 & 89  & 13,058 & coverage \\
		Product $R_Y$ & 10 & 2,044 & 100 & 2 & 2 & 1.000000 & 109 & 19,902 & coverage \\
		\midrule
		Grid-CZ & 4  & 28    & 16  & 6 & 2 & 0.990118 & 26  & 494    & coverage \\
		Grid-CZ & 5  & 60    & 25  & 4 & 2 & 0.992181 & 50  & 1,221  & coverage \\
		Grid-CZ & 6  & 124   & 36  & 3 & 2 & 0.994095 & 77  & 2,536  & coverage \\
		Grid-CZ & 7  & 252   & 49  & 3 & 2 & 0.990438 & 111 & 4,706  & coverage \\
		Grid-CZ & 8  & 508   & 64  & 3 & 2 & 0.987322 & 151 & 8,010  & resources \\
		Grid-CZ & 9  & 1,020 & 81  & 2 & 2 & 0.999878 & 196 & 12,872 & coverage \\
		Grid-CZ & 10 & 2,044 & 100 & 2 & 2 & 0.999874 & 248 & 19,615 & coverage \\
		\bottomrule
	\end{tabular}
\end{table*}

\subsubsection{Incremental observational coverage}

Relative to each run's fixed reference horizon, the product embedding reached
the 0.99 reference-coverage criterion at support 3 for
\(L=4,5,6\) and at support 2 for \(L=9,10\).
For \(L=7\) and \(L=8\),
\(C_{\mathrm{ref}}(2)\)
was 0.978420 and 0.977275, respectively, and the next complete downstream
grade was rejected by the declared processing budget. These endpoints are
therefore resource-bounded and must not be interpreted as geometric plateaus
or convergence.

For the grid-CZ embedding, support 2 reached the 0.99 reference-coverage
criterion for every evaluated size except \(L=8\), where
\(C_{\mathrm{ref}}(2)=0.987322\)
and the next complete downstream grade was rejected by the configured resource
policy.

The accumulated retained geometric dimension is the number of independent
directions preserved in the normalized observable-response geometry after
applying both the declared numerical-resolvability criterion and the
\(\tau_{\mathrm{dir}}\) component-retention criterion. In QILP-0 it is used
as a compact descriptor of observational structural richness; it is not a
measure of physical energy, Shannon information, or intrinsic circuit
complexity. At $L=10$, the product and grid-CZ embeddings achieved almost
identical coverage---1.000000 and 0.999874, respectively---while their
accumulated retained geometric dimensions were 109 and 248. Comparable
coverage can therefore coexist with substantially different latent
organization within the evaluated observational scope.

\begin{figure}[!htbp]
	\centering
	\includegraphics[width=1\linewidth]{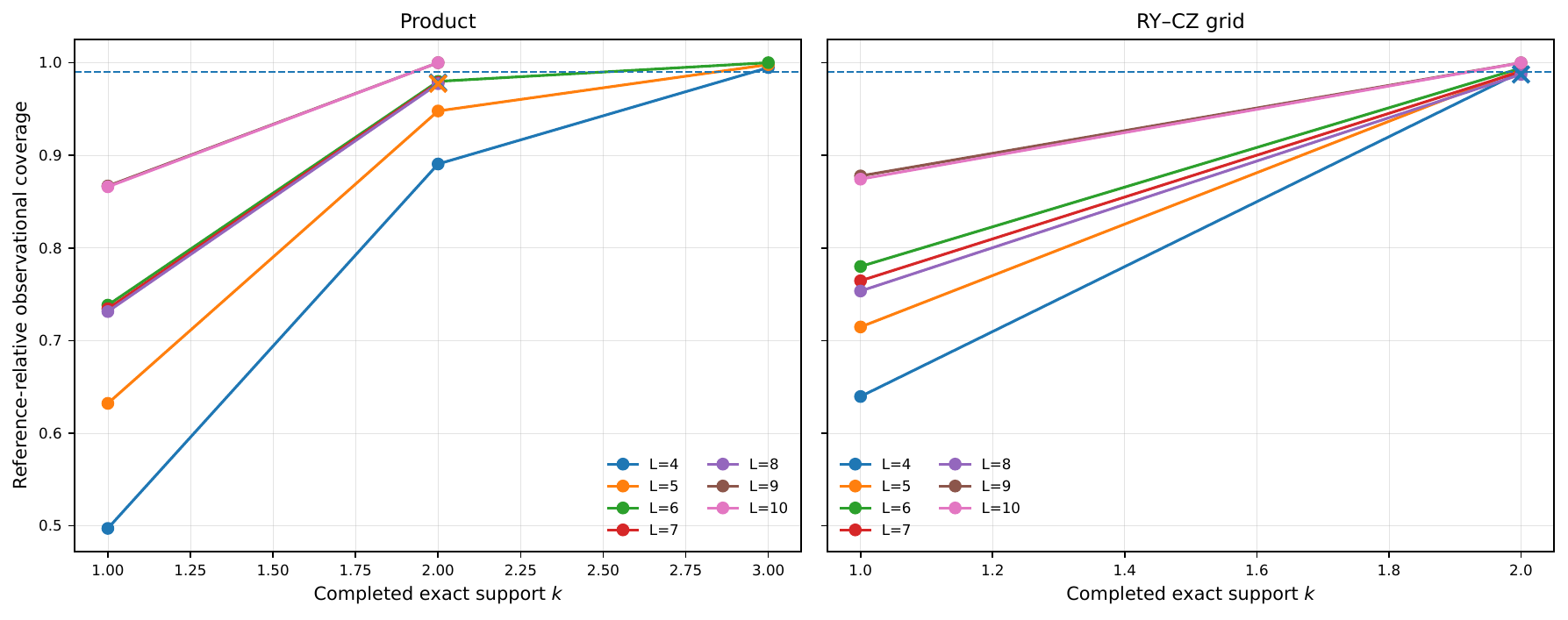}
	\caption{Incremental reference-relative observational coverage for the
		product and grid-CZ embeddings. Curves include only completed
		exact-support layers within each run's declared reference horizon. The
		dashed line marks the 0.99 criterion; crosses identify runs whose next
		candidate support was analyzed geometrically but rejected by the
		downstream post-selection feature budget.}
	\label{fig:bs-global-coverage}
\end{figure}

Figure~\ref{fig:bs-global-coverage} summarizes the incremental
reference-relative observational coverage of both embeddings across the
evaluated system sizes. The curves make explicit both the support at which the
coverage criterion is reached and the runs in which the next candidate support
was geometrically analyzed but could not be admitted as a completed
downstream layer under the declared post-selection feature budget.

\paragraph{Meaning of the resource-bounded endpoints.}
The endpoint label ``resources'' does not denote failure of the quantum
observable evaluation itself, nor does it mean that available RAM was
exhausted. In the three affected runs---product
\(L=7\), product \(L=8\), and grid-CZ \(L=8\)---the next exact-support block
was analyzed geometrically, but the set of original observables selected from
that block would have exceeded the declared downstream cumulative feature
budget.

The downstream policy allowed at most 100,000 selected original observables
and 250,000,000 accumulated selected values. For product \(L=7\), the next
support would have produced 151,936 selected observables cumulatively while
requiring 38,287,872 accumulated values. For product \(L=8\), the
corresponding candidate counts were 340,946 observables and 173,200,568
values. For grid-CZ \(L=8\), support 3 would have produced 309,549 selected
observables cumulatively and 157,250,892 accumulated values. Thus, in all
three runs the binding constraint was the 100,000-observable cumulative
feature limit, whereas the accumulated-value limit remained unexceeded.

These downstream endpoints are distinct from the prospective resource policy
used to declare
\(K_{\mathrm{ref}}\).
A support may therefore contribute to the fixed reference calibration while
not becoming a completed downstream layer. In the affected runs,
\(K_{\mathrm{end}}=2\)
although the declared reference horizon extended to support 3. Their reported
coverages are consequently reference-relative values at the last completed
downstream support, not estimates of unrestricted higher-support Pauli
coverage.

\subsubsection{From continuous geometry to declarative reconstruction}

This experiment focuses on quantum embeddings, which, as discussed in the
motivation, constitute a distinct source of opacity in QML pipelines. The two
circuits transform exactly the same classical inputs into different quantum
representations, and QILP-0 is used to examine how this architectural change
is reflected first in the observable geometry and subsequently in the induced
declarative relation.

After discretization, PRIDE induces a propositional logic theory describing
the relation between the selected observable states and the Bars/Stripes
target over the analysed dataset. The purpose of this theory is declarative
reconstruction rather than predictive classification. Nevertheless, evaluating
the induced program on the same observed relation provides a direct operational
check of its correctness and completeness. We therefore report strict
reconstruction accuracy as the fraction of analysed rows for which the theory
recovers the observed target without ambiguity. This quantity measures
equivalence to the analysed relation and must not be interpreted as
out-of-sample predictive accuracy.

In this context we call an induced theory stable across successive accumulated
supports when adding the newly admitted observable block leaves the reported
deterministic PRIDE rule set unchanged.

Under this criterion, the product embedding was already stable at support 1
for every evaluated size.
For the grid-CZ embedding, by contrast, the rule set changed between
supports 1 and 2 for $L=4,5,7,8$, while it remained unchanged for
$L=6,9,10$.

Thus, geometric refinement and declarative refinement need not
coincide: additional observational directions may be incorporated without
altering the final rule theory, whereas the entangling layer makes the declarative reconstruction
support-dependent for some system sizes.

The $L=5$ case provides a compact example of the declarative object produced
by QILP-0. In this dataset, \texttt{target(0)} denotes Bars and
\texttt{target(1)} denotes Stripes. A literal \texttt{Xj(s)} states that the
original single-qubit Pauli observable $X$ on qubit $q_j$ takes native
discrete state $s$. Qubits $q_0,\ldots,q_4$ correspond, from left to right,
to the first image row.

\paragraph{Product embedding.}
For this embedding, each $X_j$ takes one of two exact native values:
\begin{equation}
	\texttt{Xj(0)} \Longleftrightarrow \langle X_j\rangle=0,
	\qquad
	\texttt{Xj(1)} \Longleftrightarrow
	\langle X_j\rangle=\sin(\pi/3).
	\label{eq:bs-product-pride-states}
\end{equation}
The complete $L=5$ theory is
\begin{quote}
	\raggedright
	\ttfamily\scriptsize
	target(0) :- X0(0), X1(1).\\
	target(0) :- X0(1), X1(0).\\
	target(0) :- X0(0), X2(1).\\
	target(0) :- X0(1), X2(0).\\
	target(0) :- X0(0), X3(1).\\
	target(0) :- X0(1), X3(0).\\
	target(0) :- X0(0), X4(1).\\
	target(0) :- X0(1), X4(0).\\
	target(1) :- X0(0), X1(0), X2(0), X3(0), X4(0).\\
	target(1) :- X0(1), X1(1), X2(1), X3(1), X4(1).\\
\end{quote}

The eight rules for \texttt{target(0)} detect a disagreement between the
first observable and at least one other observable in the first row. This is
the expected signature of a non-uniform row and therefore of a Bars pattern,
once the two globally uniform images have been excluded. The two
\texttt{target(1)} rules identify the two possible uniform first rows,
corresponding to Stripes. The rule set is identical for accumulated supports
$\leq1$ and $\leq2$: the additional support-2 observational structure does not
change the declarative reconstruction in this case.

\paragraph{Grid-CZ embedding.}
For grid-CZ, the $L=5$ accumulated support-$\leq2$ theory contains 14 rules.
Compared with the support-$\leq1$ theory, 13 rules are unchanged
and only one Stripes condition is replaced:
\begin{quote}
	\raggedright
	\ttfamily\scriptsize
	\% support $\leq 1$\\
	target(1) :- X1(1), X3(1).\\[2pt]
	\% support $\leq 2$\\
	target(1) :- X0(1), X3(1).\\
\end{quote}

The change is therefore localized rather than a wholesale replacement of the
theory. It also illustrates an important traceability property: although the
support-$\leq2$ symbolic dataset contains observables up to support 2, the
final $L=5$ rule bodies remain expressed through original support-1
observables. For both embeddings and both accumulated supports, the theories
cover all 60 exhaustive $L=5$ patterns, produce no conflicting predictions,
and achieve strict reconstruction accuracy 1.0.

\subsubsection{Observational-twin certificate}
\label{sec:bs-twin-certificate}

The Bars \& Stripes executions provide the native-discrete case of the
constructive QILP-0 result. Across the reported product-embedding runs
($L=4,\ldots,10$), 17 accumulated-support layers were reported; across the
grid-CZ runs, 14 layers were reported.
Every selected feature in both
families was audited as an exact native discrete observable state. Therefore,
the discretizer preserved the observed states and introduced no numerical
cuts. Observable generation, ordering, geometric analysis and discretization
remained structural and target-independent.

Table~\ref{tab:bs-twin-certificate} summarizes the symbolic certificate. Every
reported relation was target-consistent, every PRIDE program covered every
analysed row, no conflicting prediction was produced, and the minimum strict
reconstruction accuracy was 1.0. Hence each reported layer supports an exact
observational-twin claim for its native-discrete relation under the declared
observational scope.

\begin{table*}[t]
	\centering
	\small
	\setlength{\tabcolsep}{4pt}
	\caption{Observational-twin certificate for the Bars \& Stripes experiment
		families. Sample-sufficiency warnings are retained as diagnostics but do
		not alter exact native states.}
	\label{tab:bs-twin-certificate}
	\begin{tabular}{lrrrrlll}
		\toprule
		Embedding & Layers & Supp. range & Continuous cols. & Max. conflicts &
		Max. uncovered & Min. strict acc. & Theory change 1$\rightarrow$2 \\
		\midrule
		Product $R_Y$ & 17 & 1--3 & 0 & 0 & 0 & 1.0 & none \\
		Grid-CZ       & 14 & 1--2 & 0 & 0 & 0 & 1.0 & $L=4,5,7,8$ \\
		\bottomrule
	\end{tabular}
\end{table*}

The audit also recorded sample-sufficiency warnings for some native variables.
These warnings are retained as diagnostic metadata. They indicate that a
conservative sample-cardinality criterion would have suggested fewer states,
but they do not alter the native states observed in the experiment or the
resulting certificate.
A systematic comparison of alternative cardinality bounds and more refined
procedures for estimating defensible symbolic resolution could remain an
interesting direction for future work.

\subsubsection{Methodological interpretation}

This experiment isolates the quantum-embedding stage as a source of
observational opacity in QML. The product circuit provides a separable
encoding, whereas the grid-CZ circuit adds local two-qubit interactions after
the same local input rotations. Because the dataset and local encoding are
held fixed, the comparison allows QILP-0 to track how this controlled change
of quantum producer modifies the normalized observable-response geometry and
whether those changes propagate to the declarative relation reconstructed from
the original observable states.

The results do not establish a universal advantage of entanglement. For this
community benchmark and the executed supports, the grid-CZ circuit can reach
nearly the same reference-relative coverage while exhibiting a larger
accumulated retained geometric dimension and, for some image sizes, a
support-dependent declarative theory.
Conversely,
additional geometric structure does not necessarily alter the rule theory.
Resource-limited endpoints remain explicitly separated from convergence
claims.

\subsection{Continuous validation: Low-Depth MNIST}
\label{sec:low-depth-mnist}

\subsubsection{Community dataset and experimental comparison}

We use the Low-Depth MNIST resource derived from the community study of
structured image loading with shallow quantum circuits
\cite{kiwit2025typical,kiwit2025lowdepthmnist,kiwit2025lowdepthdemo}.
The underlying images belong to MNIST \cite{lecun1998gradient}, while the
target states use the flexible representation of quantum images (FRQI)
\cite{le2011frqi}.

The experiment retains all 14,708 published instances labelled as digits 0
and 1: 6,912 zeros and 7,796 ones. Each image is represented on 11 qubits.
Qubit $q_0$ is the FRQI colour qubit and $q_1,\ldots,q_{10}$ form the address
register. The depth-4 preparation associates each image with 171
input-dependent rotation parameters and a fixed layout of 251 operations:
171 $R_Y$ rotations and 80 CNOT gates. Its mean fidelity with the exact FRQI
state is 0.968639 (median 0.978554). Fidelity is retained only as
representation metadata and is not provided to the symbolic learner.

We compare two representations of the same ordered rows:
\begin{equation}
	\mathbf{A}:\qquad |\psi_{d4}(x)\rangle,
\end{equation}
and
\begin{equation}
	\mathbf{B}:\qquad
	U_{\mathrm{VQC}}(\theta^\star)|\psi_{d4}(x)\rangle.
\end{equation}

Representation $\mathbf{A}$ is the published depth-4 preparation, hereafter
the d4 preparation. Representation $\mathbf{B}$ applies the trained
variational classifier to the same prepared states.

Figure~\ref{fig:mnist-circuit-architecture} summarizes the executed circuit
architecture and provides gate-level details of both the d4 preparation and
the variational stage.

\begin{figure*}[!htbp]
	\centering
	\includegraphics[width=\textwidth]{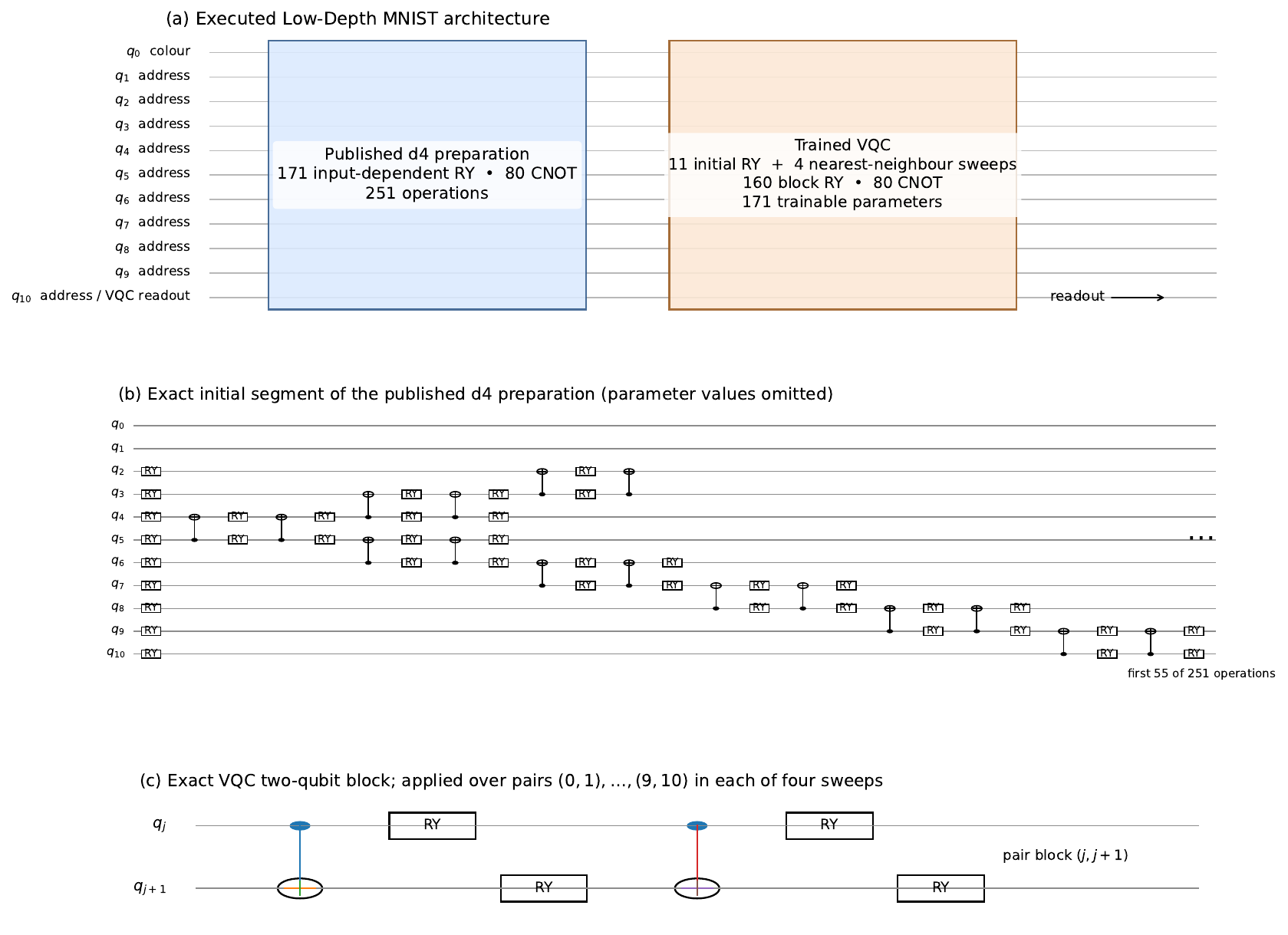}
	\caption{Executed Low-Depth MNIST architecture used for the before/after
		comparison. Panel (a) summarizes the 11-qubit processing chain: the
		published depth-4 (d4) preparation is followed by the trained variational
		classifier. Qubit $q_0$ is the FRQI colour qubit,
		$q_1,\ldots,q_{10}$ form the address register, and $q_{10}$ is also used
		as the VQC readout qubit. Panel (b) shows an exact initial segment of the
		executed d4 preparation; gate ordering is preserved while numerical
		rotation angles are omitted for readability. The complete d4 layout
		contains 171 input-dependent $R_Y$ rotations and 80 CNOT gates.
		Panel (c) shows the exact two-qubit VQC block
		\texttt{CNOT--RY--RY--CNOT--RY--RY}, which is applied over the ten
		nearest-neighbour pairs in each of four sweeps. The complete gate-level
		circuit is retained with the reproducibility materials.}
	\label{fig:mnist-circuit-architecture}
\end{figure*}

\subsubsection{Variational circuit and supervised training}

The classifier follows the small linear VQC used in the community demo
\cite{kiwit2025lowdepthdemo}. It starts with one local $R_Y$ rotation on every
qubit and then performs four sequential sweeps over the ten neighbouring qubit
pairs. Each two-qubit block is decomposed as
\begin{center}
	\texttt{CNOT--RY--RY--CNOT--RY--RY}.
\end{center}
The circuit therefore has
\begin{equation}
	11+4\times10\times4=171
\end{equation}
shared trainable parameters. These parameters are distinct from the 171
image-dependent parameters of the depth-4 preparation. The classifier reads
the computational-basis probabilities of $q_{10}$; this qubit belongs to the
address register in the image preparation and becomes the readout qubit of the
VQC.

Training uses five epochs, batches of 128 rows, Adam with learning rate 0.01,
and a fixed 80/20 split with seed 0. The split contains 11,766 training rows
and 2,942 validation rows. The validation accuracy of the stored parameters
increases from 0.465330 before training to 0.971108 after training. This
quantity documents that the variational transformation has learned the
supervised task; it is conceptually distinct from the strict reconstruction
accuracy later used to certify the QILP-0 declarative theory.

Because the VQC is trained with class labels, representation $\mathbf{B}$
contains supervised information. Once the trained parameters are fixed,
however, the subsequent observable evaluation, geometric analysis, structural
ordering, and discretization do not use the labels.

\subsubsection{Exact observable evaluation and executed supports}

Both representations are analysed with the same complete Pauli families at
exact supports 1, 2 and 3. All 14,708 rows are retained and no observable
columns are sampled. Expectation values are computed from exact simulated
states, so the reported matrices contain neither finite-shot uncertainty nor
hardware noise.

At exact supports 1, 2 and 3, the evaluated families contain respectively 22,
275 and 2,310 non-constant observable columns. Each complete block is processed
once and the original observables selected by the geometric stage are
accumulated for subsequent discretization and declarative induction.

\subsubsection{Incremental observational coverage}

\begin{table*}[t]
	\centering
	\small
	\setlength{\tabcolsep}{4pt}
	\caption{Observable geometry before and after variational training. Coverage
		is accumulated over complete exact-support blocks and interpreted relative
		to the declared observational reference.}
	\label{tab:mnist-observational-results}
	\begin{tabular}{rrrrrrr}
		\toprule
		Support & Coverage d4 & Coverage d4+VQC & Geom. dim. d4 &
		Geom. dim. d4+VQC & Selected d4 & Selected d4+VQC \\
		\midrule
		1 & 0.355399 & 0.667402 & 21 & 19 & 22 & 22 \\
		2 & 0.809468 & 0.901905 & 237 & 257 & 260 & 285 \\
		3 & 0.999805 & 0.999899 & 1,819 & 2,092 & 2,168 & 2,498 \\
		\bottomrule
	\end{tabular}
\end{table*}

Table~\ref{tab:mnist-observational-results} shows that the trained
representation places a substantially larger fraction of the evaluated
normalized observational geometry within the low-support accumulated
subspaces. Support 1 coverage increases from 0.355399 to 0.667402, and
accumulated support 2 coverage increases from 0.809468 to 0.901905. Thus, for
this before/after comparison, a larger fraction of the declared observational
structure is represented at lower support after training.

The final coverages are almost identical, but the accumulated retained
geometric dimension rises from 1,819 to 2,092 and the number of selected
original observables from 2,168 to 2,498. Earlier concentration of coverage
at low support therefore coexists with a larger retained geometric dimension
over the complete evaluated support-3 horizon. This is a statement about the
normalized observable-response geometry of this experiment, not a general
measure of quantum-circuit complexity.

\begin{figure*}[!htbp]
	\centering
	\begin{minipage}[t]{0.32\textwidth}
		\centering
		\includegraphics[width=\linewidth]{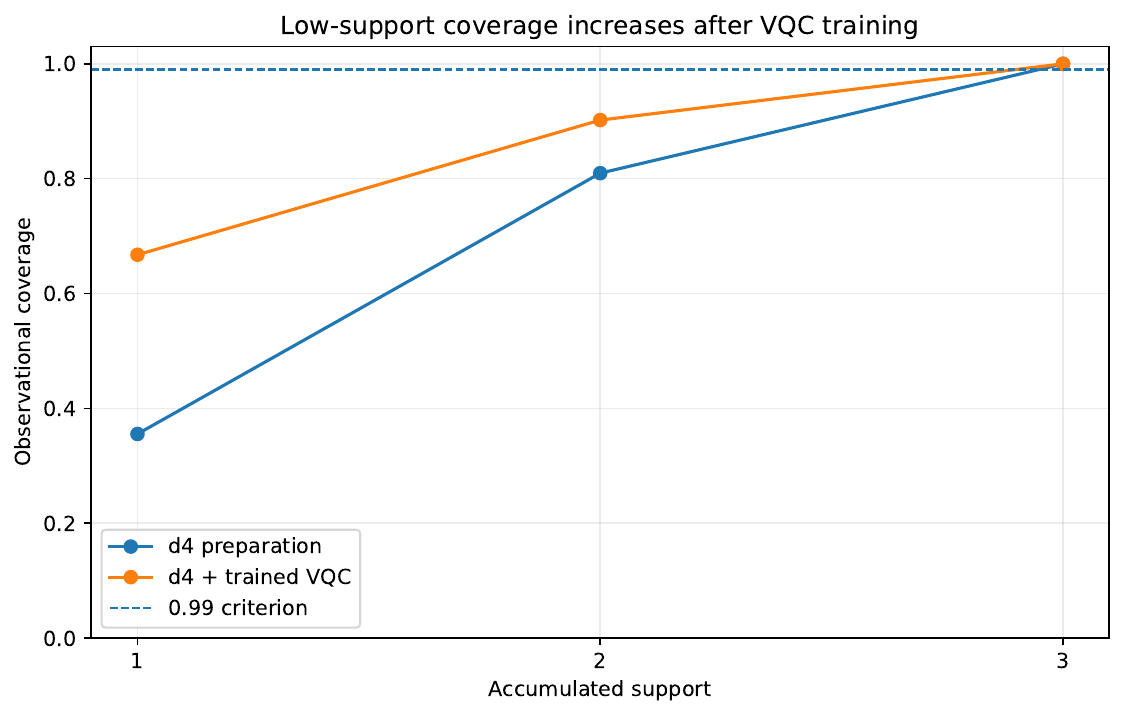}
	\end{minipage}
	\hfill
	\begin{minipage}[t]{0.32\textwidth}
		\centering
		\includegraphics[width=\linewidth]{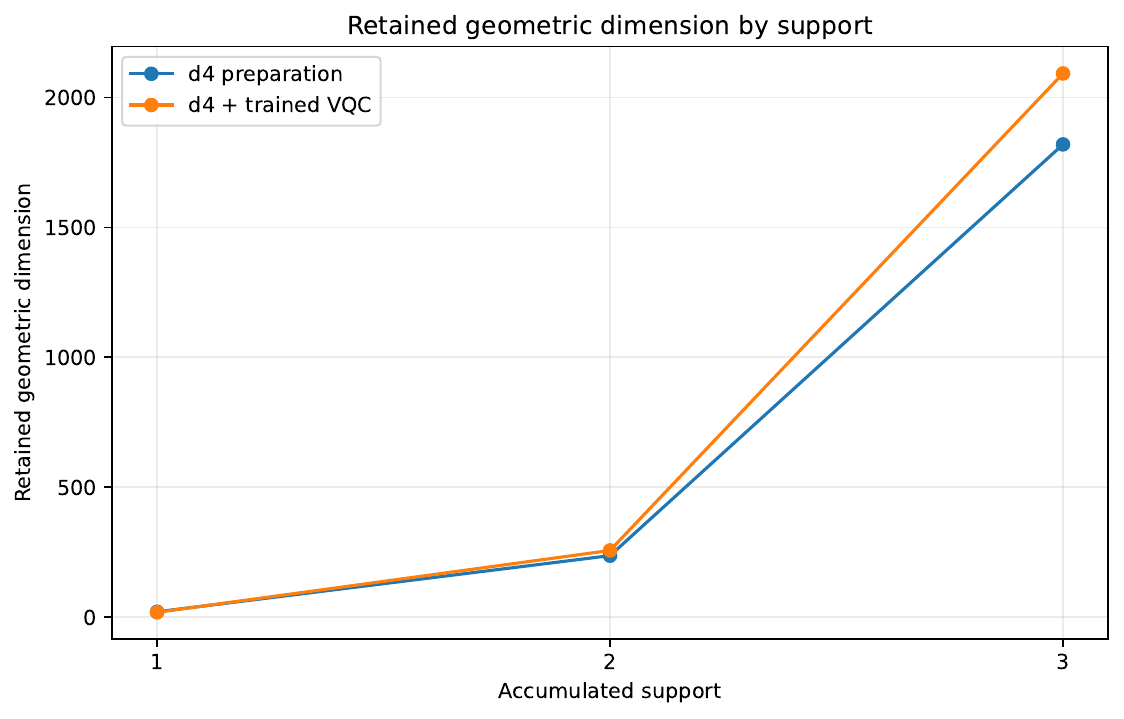}
	\end{minipage}
	\hfill
	\begin{minipage}[t]{0.32\textwidth}
		\centering
		\includegraphics[width=\linewidth]{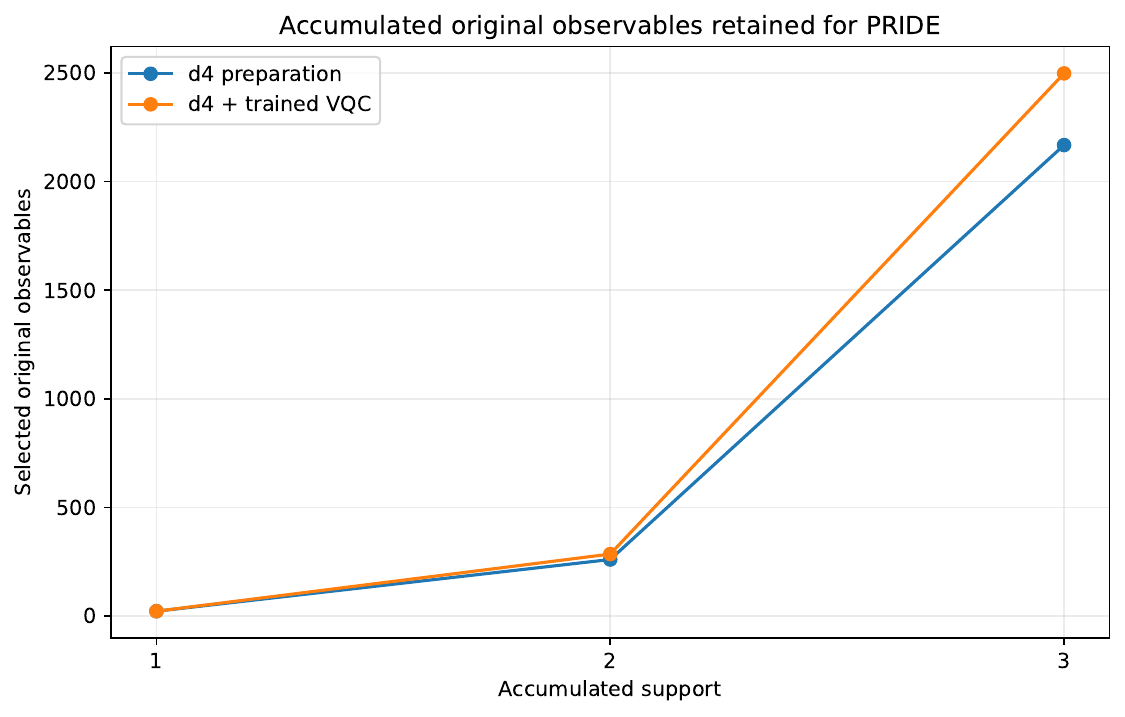}
	\end{minipage}
	\caption{Before/after observational geometry. Left: accumulated coverage over
		complete support blocks. Centre: accumulated retained geometric dimension.
		Right: accumulated number of original observables retained for the symbolic
		stage. The trained representation accumulates more coverage at supports 1
		and 2 while exhibiting a larger retained geometric dimension and
		selected-observable vocabulary at support 3.}
	\label{fig:mnist-geometry-before-after}
\end{figure*}

\subsubsection{Declarative reconstruction and theory structure}

PRIDE receives the discretized original observables selected by QILP-0. In
every evaluated accumulated-support layer, the induced theory is complete and
conflict-free over all 14,708 analysed rows and achieves strict reconstruction
accuracy 1.0. This measures exact reconstruction of the corresponding finite
discrete observational relation and is distinct from the VQC validation
accuracy of 0.971108, which measures predictive performance on the fixed
held-out split.

\begin{table*}[t]
	\centering
	\small
	\setlength{\tabcolsep}{4pt}
	\caption{Symbolic results before and after variational training. ``Available''
		denotes the number of selected original observable columns presented to
		the discretization and symbolic stages at each accumulated support.}
	\label{tab:mnist-symbolic-results}
	\begin{tabular}{rrrrrrr}
		\toprule
		Support & Available d4 & Rules d4 & Literals d4 &
		Available d4+VQC & Rules d4+VQC & Literals d4+VQC \\
		\midrule
		1 & 22    & 1,106 & 2,300 & 22    & 2,749 & 6,147 \\
		2 & 260   & 1,111 & 2,309 & 285   & 2,651 & 5,913 \\
		3 & 2,168 & 1,106 & 2,301 & 2,498 & 2,361 & 5,298 \\
		\bottomrule
	\end{tabular}
\end{table*}

The resulting programs provide the declarative reconstructions analysed below.
Despite the large available vocabularies in
Table~\ref{tab:mnist-symbolic-results}, all literals of the final support-3
theories use only

\begin{equation}
	X_0,\qquad Z_0,\qquad X_{10},\qquad Z_{10}.
\end{equation}

Thus, the final rule bodies are expressed through observables of the FRQI
colour qubit $q_0$ and of $q_{10}$, the address qubit used as the VQC readout.
This does not imply that the intermediate qubits are irrelevant to the circuit
or to its continuous observational geometry; it states only that their
observables do not occur in the final rule bodies.

The complete MNIST theories are too large to reproduce in the article, but
representative rules make their syntax and provenance explicit. A literal such
as \texttt{X0(25)} denotes equality of the original observable $X_0$ to one of
its independently fitted discrete states; it is not a learned numerical
threshold. Representative rules from the final support-3 theories are:
\begin{quote}
	\raggedright
	\ttfamily\scriptsize
	\% d4 preparation\\
	target(0) :- X0(25).\\
	target(1) :- X0(8).\\
	target(0) :- X0(17), Z0(22).\\
	target(1) :- X10(103), Z0(32).\\[3pt]
	\% d4 + trained VQC\\
	target(0) :- Z10(16).\\
	target(1) :- Z10(8).\\
	target(0) :- X0(16), Z0(14).\\
	target(1) :- X10(56), Z0(53).\\
\end{quote}

A direct comparison of raw rule counts requires care because discretization is
fitted independently to the two continuous representations. The VQC changes
the observable profiles and therefore can change the number and occupancy of
discrete states assigned to a given observable. Table~\ref{tab:mnist-rule-structure}
shows that the larger post-VQC theory cannot be attributed to a simple
expansion of the discrete vocabulary. Across the four observables that occur in
the final theories, the total number of available states decreases from 350 to
216 and the number of distinct observable--state atoms actually used decreases
from 244 to 215, while the number of rules increases from 1,106 to 2,361.

\begin{table}[t]
	\centering
	\small
	\caption{Structural comparison of the final support-3 PRIDE theories. An
		observable-level schema is a distinct pair consisting of the target class
		and the set of observable names in the rule body after ignoring their
		discrete state identifiers. ``Available states'' is summed over the four
		observables that actually occur in the final theories.}
	\label{tab:mnist-rule-structure}
	\begin{tabular}{lrr}
		\toprule
		Quantity & d4 & d4 + trained VQC \\
		\midrule
		Rules                                  & 1,106 & 2,361 \\
		Rule literals                          & 2,301 & 5,298 \\
		Unary rules                            & 26    & 32 \\
		Binary rules                           & 965   & 1,721 \\
		Ternary rules                          & 115   & 608 \\
		Mean body length                       & 2.08  & 2.24 \\
		Available states over \\ used observables & 350   & 216 \\
		Distinct discrete atoms used           & 244   & 215 \\
		Observable-level schemas               & 21    & 24 \\
		\bottomrule
	\end{tabular}
\end{table}

The growth of the post-VQC theory is therefore better described as a more
extensively instantiated discrete relation over a closely related small
observable vocabulary. The increase is concentrated mainly in binary and
ternary rule instances, while the number of observable-level schemas changes
only from 21 to 24. These symbolic counts characterize the discretized
relation produced after independently fitted agnostic discretization; they are
not direct measures of the computational cost or intrinsic complexity of the
quantum circuit.

\subsubsection{Admissible-discretization and observational-twin certificate}
\label{sec:mnist-twin-certificate}

Low-Depth MNIST exercises the continuous discretization branch of QILP-0. The d4 and
post-VQC representations were evaluated exactly for all 14,708 rows.
Consequently, the backend uncertainty term is zero and no uncertainty-derived
SNR bound is required in these runs. All selected observable columns were
discretized independently of the target according to the admissible
observable-wise contract defined above.

Table~\ref{tab:mnist-discretization-certificate} summarizes the
discretization audit.
Across all six accumulated-support datasets, every
selected column was treated as continuous, no selected column was constant, no
requested cut set collapsed numerically, every effective bin was occupied, and
the smallest observed bin occupancy remained positive. The large number of
sample-sufficiency warnings follows the current warning-only policy: these
flags record a conservative diagnostic and do not imply empty bins, loss of
totality, or target-dependent modification of the discretizer.

\begin{table*}[t]
	\centering
	\small
	\caption{Admissible-discretization certificate for Low-Depth MNIST. ``Warn.''
		denotes columns carrying the sample-sufficiency audit warning;
		``Min. occ.'' is the minimum effective-bin occupancy over all selected
		columns.}
	\label{tab:mnist-discretization-certificate}
	\begin{tabular}{llrrrrrr}
		\toprule
		Representation & Support $\leq$ &
		Columns & Warn. & Effective bins & Empty bins & Collapsed columns &
		Min. occ. \\
		\midrule
		d4 preparation     & 1 & 22    & 20    & 2,176   & 0 & 0 & 79 \\
		d4 preparation     & 2 & 260   & 256   & 28,703  & 0 & 0 & 43 \\
		d4 preparation     & 3 & 2,168 & 2,155 & 256,706 & 0 & 0 & 23 \\
		d4 + trained VQC   & 1 & 22    & 18    & 1,344   & 0 & 0 & 127 \\
		d4 + trained VQC   & 2 & 285   & 246   & 18,555  & 0 & 0 & 106 \\
		d4 + trained VQC   & 3 & 2,498 & 2,279 & 175,554 & 0 & 0 & 76 \\
		\bottomrule
	\end{tabular}
\end{table*}

The post-discretization audit establishes twin- admissibility at every support
for both representations: no two identical discrete observational states are
associated with incompatible targets. PRIDE then covers all 14,708 rows,
produces no conflicting predictions, and achieves strict reconstruction
accuracy 1.0 in every accumulated-support layer. The certificate therefore
supports exact equivalence between each induced theory and its discrete
observational relation.

This statement does not claim that finite discretization preserves every
metric distinction of the original continuous matrices. Rather, it states that
the target-independent discretization satisfies the declared observational
resolution contract and that the induced declarative theory exactly
reconstructs the finite relation produced under that contract.

\subsubsection{Methodological interpretation}

The experiment separates three effects that should not be conflated. First,
supervised VQC training increases validation accuracy on the fixed held-out
split from 0.465330 to 0.971108. 
Second, the trained representation accumulates
a substantially larger fraction of the evaluated observational geometry at
supports 1 and 2 while also exhibiting a larger retained geometric dimension
at support 3.
Third,
the corresponding discrete relation is represented by a larger and more
extensively instantiated declarative theory, although its final rule bodies
remain restricted to observables of $q_0$ and $q_{10}$.

These observations are specific to the executed Low-Depth MNIST case. 
Their role here is to
show that the same QILP-0 construction can audit and exactly reconstruct two
continuous observational relations located before and after a supervised
quantum transformation.

\subsection{Cross-experiment certification}
\label{sec:cross-experiment-twin-certificate}

Table~\ref{tab:cross-experiment-twin-certificate} collects the common
certificate fields across the four reported experimental routes. All of them
use exact observable evaluation, target-independent construction of the
symbolic vocabulary, complete row-level audits, and exact PRIDE reconstruction
of the resulting discrete relation. Bars \& Stripes exercises preservation of
native discrete states, whereas Low-Depth MNIST exercises the continuous
admissible-discretization contract.

\begin{table*}[t]
	\centering
	\small
	\setlength{\tabcolsep}{4pt}
	\caption{Cross-experiment observational-twin certificate. ``Layers'' denotes
		reported accumulated-support relations.}
	\label{tab:cross-experiment-twin-certificate}
	\begin{tabular}{lclrrrrr}
		\toprule
		Case & Layers & Discret. & Const.cols. & Empty bins/states &
		Max.conflicts & Max.uncov. & Min.strict acc. \\
		\midrule
		B\&S product      & 17 & Native states        & 0 & 0 & 0 & 0 & 1.0 \\
		B\&S grid-CZ      & 14 & Native states        & 0 & 0 & 0 & 0 & 1.0 \\
		MNIST d4          & 3  & Continuous agnostic  & 0 & 0 & 0 & 0 & 1.0 \\
		MNIST d4 + VQC    & 3  & Continuous agnostic  & 0 & 0 & 0 & 0 & 1.0 \\
		\bottomrule
	\end{tabular}
\end{table*}

Within the finite observational scope declared for each run, these results
support exact observational declarative twins for the corresponding reported
discrete relations. Coverage metadata are reported separately relative to each
run's declared observational reference horizon.

\subsection{Interpreting certificates beyond the exact reported cases}
\label{sec:certificate-interpretation}

The experiments above use exact observable evaluations, but exactness of the
symbolic certificate and exactness of the underlying observable evaluation are
logically distinct. If a backend reports estimated observable values together
with uncertainty or effective-resolution metadata, those values can still be
discretized into a target-consistent relation and reconstructed exactly by the
induced theory. In that situation, the declarative equivalence remains exact
with respect to the reported discrete relation, while the correspondence
between the reported responses and ideal quantum responses is qualified by the
reported backend or provider uncertainty.

A different situation occurs when twin-admissibility fails. If two analysed
rows map to the same discrete observational state but carry incompatible
targets, QILP-0 still constructs and preserves the finite observed relation and
records the conflicting states. However, the relation is not deterministic
with respect to the target and no exact observational-twin claim is made for
that configuration. In the exact-twin branch, this consistency audit
terminates the construction before PRIDE induction while retaining the
non-deterministic relation and its provenance as an auditable result.

Likewise, a resource-limited observational traversal changes the declared
scope; it does not by itself make the logical reconstruction approximate.

A future notion of a certified approximate observational twin would require an
explicit approximation contract defining which discrepancies are admitted and
how they are bounded---for example, at the observational, discretization, or
symbolic-reconstruction level. No such additional contract is required by, or
claimed for, the exact experiments reported here.

\subsection{Additional implementation validation}
\label{sec:internal-validation}

Before consolidating the two reported case studies, we performed additional
implementation checks on heterogeneous circuit families and physical-model
instances. These included circuits derived from MQT Bench, small TFIM and
Heisenberg configurations, and product and entangling Bars-and-Stripes
embeddings over increasing circuit sizes. The checks exercised exact-support
observable generation, incremental accumulation without regenerating lower
supports, target-independent geometric processing, coverage and stopping
metadata, agnostic discretization, PRIDE induction, provenance and output
handling, and backend consistency.
Increasing-size executions were also used to expose
practical resource limits and to identify which stages became computationally
restrictive under the tested configurations.

These executions were used to identify computational limits and to verify the
main invariants of the prototype across different circuit structures. They are
not presented as an additional comparative benchmark because the
configurations were exploratory and were not all executed under a single
controlled protocol. Quantitative experimental claims in this paper are
therefore restricted to the Bars \& Stripes and Low-Depth MNIST studies above.

\section{Conclusions and Future Work}
\label{sec:conclusions}

This work has investigated whether the observable behaviour of a quantum
circuit can be transformed into a finite declarative object without allowing
the target to determine the observational vocabulary and without losing the
semantic link between the resulting logical literals and the original quantum
observables. QILP-0 provides a constructive affirmative answer under an
explicit set of methodological conditions and within a declared finite
observational scope.

The resulting object was defined as an \emph{observational declarative twin}.
Its exactness is assessed with respect to the finite task-conditioned discrete
relation produced under the declared observational scope.

The construction separates two layers. The continuous layer generates and
organizes semantically traceable observable profiles, traverses them according
to a reproducible structural grading, analyses their target-independent
geometry, and measures reference-relative incremental coverage against a fixed
observational reference. Although SVD is used to organize the numerical
geometry in the present realization, latent coordinates do not replace the
original observable vocabulary. The retained new subspace is mapped
deterministically back to original observable columns before the symbolic
boundary. The declarative layer then preserves native discrete states or
performs admissible target-independent discretization of continuous profiles,
audits twin-admissibility, and uses LFIT/PRIDE to induce and certify the
corresponding finite logic program. This separation makes the provenance chain
from a rule literal back to its observable, support, and qubits an explicit
part of the result.

The two experimental studies exercised complementary branches of this
construction. In Bars \& Stripes, exhaustive native-discrete datasets were
analysed for product and grid-CZ embeddings over systems ranging from 16 to
100 qubits. 
The experiment showed that similar accumulated coverage can
coexist with different retained geometric dimensions and that geometric
refinement does not necessarily imply a change in the induced declarative
theory.
Every
reported native-discrete relation satisfied the exact-twin certificate.

Low-Depth MNIST exercised the continuous branch on all 14,708 digit-0/1
instances. Comparing the published depth-4 preparation with the same states
after a trained VQC showed that the trained representation placed a larger
fraction of its declared support-3 reference geometry within the accumulated
support-1 and support-2 subspaces, while also producing a larger retained geometric dimension at support 3 and a
more extensively instantiated declarative relation. In
both representations and at every reported accumulated support, the
discretized relation remained twin-admissible and the induced theories covered
all analysed rows without conflicts, yielding strict reconstruction accuracy
equal to one.

These results support three conclusions. First, a useful symbolic account of a
quantum representation need not be based on latent symbolic variables:
continuous latent geometry can guide a reproducible selection while the final
declarative vocabulary remains tied to original observables. Second,
geometric and declarative refinement are distinct phenomena; additional
observational directions may be detected without forcing a corresponding
change in the logical theory. Third, exactness can be stated meaningfully at
the symbolic level provided that the observational reference, discretization,
finite relation, and provenance under which that exactness is claimed are
themselves made explicit.

Several limitations delimit the present contribution. QILP-0 is currently an
order-0 construction over a finite multi-valued propositional representation.

Its certificates apply to finite datasets and declared observational reference
horizons, not to unrestricted circuit behaviour. For the Pauli realization,
the structural maximum is
\(K_{\mathrm{phys}}=n\),
whereas a particular execution may declare
\(K_{\mathrm{ref}}<K_{\mathrm{phys}}\).
In that case, the coverage trajectory is explicitly reference-relative to
\(K_{\mathrm{ref}}\),
and the resulting declarative certificate is scoped to the finite relation
constructed under that reference. No percentage of the unobserved
higher-support geometry is inferred.

The Pauli exact-support realization
also faces the combinatorial growth described in
Section~\ref{subsec:qilp-computational-scope}: the declared support-bounded
family is polynomial in the number of qubits only when the maximum support is
treated as fixed, whereas an unrestricted traversal of the full Pauli family
grows exponentially.

The 100-qubit Bars \& Stripes executions reported here exploit
embedding-specific algebraic structure and the corresponding specialized
observable providers. They must therefore not be interpreted as evidence that
arbitrary 100-qubit circuits can be observed with the same computational cost.

More generally, the reference horizon may be supplied by scientific scope,
available observational data, provider capability, or a technological policy.
When resource assessment is used to choose it, practical limits depend jointly
on the observable provider, dataset size, execution environment, and declared
budgets. A separate downstream policy may still produce
\(K_{\mathrm{end}}<K_{\mathrm{ref}}\)
after the reference has been fixed.

A further limitation is that the two main experimental studies use exact
observable evaluation. The framework already distinguishes backend/provider
uncertainty from logical reconstruction, but a systematic experimental
assessment under finite shots, device noise and hardware-specific error models
remains to be performed. Likewise, the current discretization contract defines
auditable resolution criteria and diagnostics, but alternative combinations
of statistical, numerical and provider-informed bounds deserve controlled
comparison.

These limitations define several directions for future work.

\begin{enumerate}
	\item \textbf{Uncertainty-aware and hardware execution.}
	We intend to evaluate QILP-0 with finite-shot, noisy and real-hardware
	providers, propagating available uncertainty and calibration information
	into the observable-resolution metadata while keeping logical exactness
	explicitly relative to the reported discrete relation.
	
	\item \textbf{Certified approximate observational twins.}
	The present work deliberately avoids calling a failed exact certificate an
	approximate twin. A future extension should define an explicit approximation
	contract, including which discrepancies are admitted at the observational,
	discretization and symbolic levels and which quantitative bounds must be
	reported.
	
	\item \textbf{Scalable and structured observable providers.}
	The observational layer should be extended with providers that exploit
	locality, circuit structure, tensor-network representations,
	classical-shadow techniques, analytically tractable embeddings or other
	domain-supported mechanisms. 

	Such providers must retain stable observable
	semantics, provenance, and an explicit interpretation of their observational
	reference horizon.

	\item \textbf{Resource-aware geometric processing.}
	Streaming, matrix-free and distributed implementations of the geometric
	stage, together with improved prospective resource assessment and adaptive stopping
	policies, can reduce the cost of handling large observational matrices.
	These improvements should preserve the reference-relative coverage semantics
	of the fixed observational reference and the distinction between acquisition
	cost and downstream geometric processing.

	\item \textbf{Discretization and resolution analysis.}
	The current observable-wise contract can be extended through systematic
	comparisons of admissible cardinality bounds, provider-aware resolution
	criteria and uncertainty-derived limits. Of particular interest is
	determining when alternative admissible criteria produce the same symbolic
	relation and when their differences materially affect the induced theory.
	
	\item \textbf{Beyond QILP-0.}
	QILP-0 uses an order-0 finite propositional layer because it provides a
	well-defined first realization of the QXymb bridge. Future QXymb
	specializations may integrate other declarative engines and more expressive
	logical representations, provided that their semantic contract, provenance
	and reconstruction guarantees are stated explicitly.
	
	\item \textbf{Links with formal verification.}
	The induced observational theory and its provenance may provide an
	interface to complementary quantum program logics and verification
	frameworks. A promising direction is to study how empirically induced
	declarative properties can be related to independently specified formal
	properties of a circuit without conflating inductive reconstruction with
	deductive verification.
	
	\item \textbf{Broader empirical validation.}
	Further experiments should cover additional QML architectures, quantum-data
	domains, observable providers and hardware platforms, with controlled
	protocols that separate methodological validation from claims about
	predictive performance or quantum advantage.
\end{enumerate}

The compact qubit-level footprint observed in the Low-Depth MNIST case also
raises two related questions for future work: whether different classification
tasks induce different declarative qubit footprints, and whether such
footprints can be used as hypotheses for task-specific circuit simplification.
The latter requires separate causal or ablation analysis, since absence from
the final rule bodies does not imply that a qubit is unnecessary for producing
the active observables.

QILP-0 should therefore be viewed as a first concrete realization of a broader
idea: quantum behaviour can be exposed through an explicitly delimited
observational interface and then translated into a declarative domain in which
equivalence, provenance and uncertainty can be inspected directly. The present
results establish this construction for a finite propositional setting and
provide the methodological basis on which richer QXymb variants can be built.

\section*{Data and code availability}

A reproducibility software package provides a minimal executable and
didactic realization of the complete QILP-0 construction, extracted from the
research implementation used in this work. It exercises the full
path from target-independent observable acquisition and geometric analysis
against a fixed observational reference to preservation of original
observables, discretization, PRIDE induction, and row-level
observational-twin certification.

To keep the complete execution directly inspectable, the software package uses
an exhaustive non-uniform \(3\times3\) Bars \& Stripes instance with a
directly declared reference horizon
\(K_{\mathrm{ref}}=2\).
This didactic configuration does not reproduce the prospective resource policy
used to choose the reference horizons of the larger reported experiments;
instead, it demonstrates the QILP-0 contract once the observational reference
horizon has been supplied.

This configuration is intentionally smaller than
the \(L=4,\ldots,10\) experiments reported in the paper and is explicitly
identified as a didactic reproducibility case rather than as an additional
quantitative benchmark. Its paper-facing outputs include the declared
observational scope, reference-relative coverage trajectory, selected original
observables, discretization audit, finite discrete relation, induced PRIDE
theory, row-level reconstruction audit, execution environment, and final twin
certificate.

The reference execution was additionally reproduced from a clean Python
environment using only the dependencies declared by the software package. The
resulting theory covers all 12 rows without conflicting predictions, achieves
strict reconstruction accuracy equal to one, and reproduces the reference
theory hash. PRIDE is accessed through the external PyLFIT 0.5.1 package and
is not vendored in the software package. The broader QXymb codebase, of which QILP-0
is the current declarative specialization developed in this work, remains
under active development.

The reproducibility package is publicly available from the project repository
at \url{https://github.com/ortegaalfonso/qilp0-reproducibility}.
Release \texttt{v1.0.0} has been permanently archived in Zenodo and is
available under the persistent identifier
\href{https://doi.org/10.5281/zenodo.22165888}
{doi:10.5281/zenodo.22165888}.

\appendix

\printcredits

\section*{Declaration of competing interest}

The authors declare that they have no known competing financial interests or
personal relationships that could have influenced the work reported in this
paper.

\section*{Funding}

This research has been funded by the European Union. Views and opinions expressed are however those of the author(s) only and do not necessarily reflect those of the European Union or ERCEA. Neither the European Union nor the granting authority can be held responsible for them.

\section*{Declaration of generative AI and AI-assisted technologies
	in the manuscript preparation process}

During the preparation of this work, the authors used OpenAI ChatGPT
to support manuscript organization, language refinement, and the drafting
and revision of explanatory text. After using this tool, the authors
reviewed and edited the content as needed and take full responsibility
for the content of the article.

\bibliographystyle{cas-model2-names}
\bibliography{QILP_refs}

\end{document}